\documentclass[journal]{IEEEtran}

\usepackage{graphicx}
\usepackage[cmex10]{amsmath}
\usepackage{amssymb}
\usepackage{cite}
\usepackage{url}
\usepackage{subfigure}
 \usepackage{epsfig}
 \usepackage{color}
 \usepackage{amsthm}

\newcommand{\ie}{\textit{i.e.}}
\renewcommand{\mid}{\text{mid}}
\renewcommand{\bar}{\overline}

\newcommand{\rev}[1]{{\color{black}{#1}}}

\newtheorem{lemma}{Lemma}
\newtheorem{corollary}{Corollary}
\newtheorem{theorem}{Theorem}

\begin{document}

\title{Active Target Tracking with Self-Triggered Communications in Multi-Robot Teams}

\author{Lifeng~Zhou,~\IEEEmembership{Student Member,~IEEE,}
        and~Pratap~Tokekar,~\IEEEmembership{Member,~IEEE}
\thanks{The authors are with the Department
of Electrical and Computer Engineering, Virginia Tech, Blacksburg,
VA, 24060 USA. (e-mail: lfzhou@vt.edu; tokekar@vt.edu).}
}

%
%



\maketitle

\begin{abstract}
We study the problem of reducing the amount of communication in decentralized target tracking. We focus on the scenario where a team of robots is allowed to move on the boundary of the environment. Their goal is to seek a formation so as to best track a target moving in the interior of the environment. The robots are capable of measuring distances to the target. Decentralized control strategies have been proposed in the past that guarantee that the robots asymptotically converge
to the optimal formation. However, existing methods require that the robots exchange information with their neighbors at all time steps. Instead, we focus on decentralized strategies to reduce the amount of communication among robots. 

We propose a self-triggered communication strategy that decides when a particular robot should seek up-to-date information from its neighbors and when it is safe to operate with possibly outdated information. We prove that this strategy converges asymptotically to the desired formation when the target is stationary. For the case of a mobile target, we use decentralized Kalman filter with covariance intersection to share the beliefs of neighboring robots. We evaluate all the approaches through simulations and a proof-of-concept experiment.

\vspace{2mm}
\emph{Note to Practitioners}---We study the problem of tracking a target using a team of coordinating robots. Target tracking problems are prevalent in a number of applications such as co-robots, surveillance, and wildlife monitoring. Coordination between robots typically requires communication amongst them. Most multi-robot coordination algorithms implicitly assume that the robots can communicate at all timesteps. Communication can be a considerable source of energy consumption, especially for small robots. Furthermore, communicating at all timesteps may be redundant in many settings. With this as motivation, we propose an algorithm where the robots do not necessarily communicate at all times, and instead choose specific triggering time instances to share information with their neighbors. Despite the limitation of limited communication, we show that the algorithm converges to the optimal configuration, both in theory as well as in simulations. 
\end{abstract}


%
\IEEEpeerreviewmaketitle

\section{Introduction}
Target tracking is one of the more well-studied problems in the robotics community~\cite{bar2004estimation} and finds many applications such as surveillance~\cite{rao1993fully,dhillon2003sensor,grocholsky2006cooperative}, crowd monitoring~\cite{tokekar2014multi,dames2015detecting}, and wildlife monitoring~\cite{dunbabin2012environmental,tokekar2013tracking}. We study \emph{active} target tracking with a team of robots where the focus is on actively controlling the state of the robot. The robots can exchange information with each other and then decide how to move, so as to best track the target. It is typically assumed that exchanging information is beneficial. It is typical to design strategies by assuming that the robots will exchange their information at each time step irrespective of whether that information is worth exchanging. In this paper, we investigate the problem of deciding when is it worthwhile for the robots to exchange information and when is it okay to use possibly outdated information.  

The motivation for our work stems from the observation that communication can be costly. For example, for smaller robots, radio communication can be a significant source of power consumption. The robots can extend their lifetime by reducing the time spent communicating (equivalently, number of messages sent). Our goal is thus to determine a strategy that communicates only when required without considerably affecting the tracking performance.

We study this problem in a simple target tracking scenario first introduced by Martinez and Bullo~\cite{martinez2006optimal}. Here, the robots are restricted to move on the boundary of a convex environment. They can obtain distance measurements towards a target moving in the interior. The goal of the robots is to position themselves so as to maximize the information gained \rev{from the target}. \rev{Our problem setup models scenarios where the robots cannot enter into the interior of the environment. For example, Pierson et al.~\cite{pierson2016cooperative} studied pursuit-evasion where the pursuers are not allowed to enter ``no-fly zones''. If the evader enters a ``no-fly zone'' then the pursuers reposition themselves on the perimeter of a convex approximation of the zone. Another motivating application is that of tracking radio-tagged fish~\cite{tokekar2013tracking} using ground robots that can move only along the boundary of the environment.}

The authors \rev{in~\cite{martinez2006optimal}} proposed a decentralized strategy where the robots communicate at all time steps with their neighbors and proved that it converges to the optimal (uniform) configuration. Instead, we apply a self-triggered coordination algorithm (following recent works~\cite{nowzari2012self,heemels2012introduction}) where each robot decides when to trigger communications with its neighbors. We apply this strategy to the aforementioned target tracking scenario and compare its performance relative to the constant strategy in simulations.

Next, we study the problem where the robots obtain noisy measurements of the distance to the targets. In a decentralized setting, robots can exchange information only with their neighbors. As a result, their local estimates of the target's position may differ considerably, resulting in poor tracking especially when the robots are not in a uniform configuration. We show how to use a decentralized Kalman filter estimator that fuses the beliefs shared by neighboring robots (at triggered instances) to a common estimate. 

\rev{Our main results assume that the robots have sufficiently large communication and sensing ranges. In Section~\ref{sec:limited}, we present necessary conditions on the sensing and communication ranges for our results to hold. We also sufficient conditions for a modified version of the self-triggered strategy to guarantee convergence.}

Simulation results validate the theoretical analysis showing that the self-triggered strategy converges to the optimal, uniform configuration. The average number of communication is less than 30\% that of the constant strategy. We also demonstrate the performance of the algorithm through proof-of-concept experiments with five simulated and two actual robots coordinating with each other.

The rest of the paper is organized as follows. \rev{We start by surveying the related works in Section~\ref{sec:related_work}. We formalize the problem in Section~\ref{sec:probform}.  The self-triggered tracking strategy is presented in Section~\ref{sec:self}, assuming that the target's position is known and is fixed. We relax these assumptions and present two  practical extensions in Section~\ref{sec:practical} for noisy measurements and limited sensing and communication ranges.} The simulation results are presented in Section~\ref{sec:sims}. We conclude with a discussion of future work in Section~\ref{sec:conc}.

A preliminary version of this paper was first presented in~\cite{zhou2017active} without the decentralized Kalman filter with covariance intersection and the analysis for the limited communication and sensing ranges (Section~\ref{sec:practical}), the Gazebo simulation experiments, and the proof-of-concept experiment (Section~\ref{sec:sims}). 

\rev{\section{Related Work} \label{sec:related_work}
Multi-robot target tracking has been widely studied in robotics~\cite{robin2016multi,khan2016cooperative}. Robin and Lacroix~\cite{robin2016multi} surveyed multi-robot target detection and tracking systems and presented a taxonomy of relevant works. Khan et al.~\cite{khan2016cooperative} classified and discussed control
techniques for multi-robot multi-target monitoring and identify the major elements of this problem. Hausman et al.~\cite{hausman2015cooperative} proposed a centralized cooperative approach for a team of robots to estimate a moving target. They showed how to use onboard sensing with limited sensing range and switch the sensor topology for effective target tracking. Dias et al.~\cite{dias2015decentralized} proposed a  multi-robot triangulation method to deal with  initialization and data association issues in  bearing-only sensors. The robot communicates locally to exchange and update the estimate beliefs of the target by a decentralized filter. Franchi et al.~\cite{franchi2016decentralized} presented a decentralized strategy to ensure that the robots follow the target while moving around it in a circle. They assume that the robots are labeled. Similar to our work, the robots attempt to maintain a uniform distribution on a (moving) circle around the target. However, unlike our work, they require that the robots constantly communicate with their local neighbors.

Sung et al.~\cite{sung2017distributed} proposed a distributed approach for multi-robot assignment problem for multi-target tracking by taking both sensing and communication ranges into account. The goal of their work is also to limit the communication between the robots. However, they do so by limiting the number of messages sent at each timestep but allow the robots to communicate at all timesteps. Instead, our work explicitly determines when to trigger communication with other robots.

Our work builds on event-triggered and self-triggered communication schemes studied primarily by the controls community~\cite{tabuada2007event,heemels2012introduction}. Dimarogonas et al.~\cite{dimarogonas2012distributed} presented both centralized and decentralized event-triggered strategies for the agreement problem in multi-agent systems. They extended the results to a self-triggered communication setting where the robot calculates its next communication time based on the previous one, without monitoring the state error. Nowzari and Cort{\'e}s~\cite{nowzari2012self} proposed a decentralized self-triggered coordination algorithm for the optimal deployment of a group of robots based on  spatial partitioning techniques. The synchronous version of this algorithm converges comparatively with an all-time communication strategy. 

To the best of our knowledge, our paper is the first to simultaneously handle both robot coordination~\cite{franchi2016decentralized} and target tracking~\cite{dias2015decentralized}. We focus on applying self-triggered control to reduce the amount of local communication between neighbors.
}

\section{Problem Formulation} \label{sec:probform}
Consider a group of $N$ robots moving on the boundary of a convex polygon $\mathcal{Q}\subset \mathbb{R}^{2}$. Let $\partial\mathcal{Q}$ denote the boundary of $\mathcal{Q}$. The robots are tasked with tracking a target with position $o$ located in the interior of $\mathcal{Q}$. Let $p_{1},...,p_{N}$ denote the positions of the robots. We can map any point on $\partial\mathcal{Q}$ to a unit circle $\mathbb{T}$ using the transformation $\varphi_{o}:\partial\mathcal{Q}\to \mathbb{T}$ given by
\begin{equation}
\varphi_{o}(p)=\frac{p-o}{\|p-o\|}
\label{eqn:transformation}
\end{equation}
 \noindent as shown in Figure~\ref{fig:Map}. We identify every robot's position with the corresponding point on the unit circle. That is, $p_{i}\in \partial\mathcal{Q}\subset\mathbb{R}^{2}$ is identified with $\theta_{i}=\varphi_{o}(p_{i})\in \mathbb{T}$, indicating the location on the circle $\mathbb{T}$ of robot $i$. Let $\theta=(\theta_{1},...,\theta_{N})\in\mathbb{T}^{N}$ denote the vector of locations of all robots. 

\begin{figure}
\centering
\includegraphics[width=0.5\columnwidth]{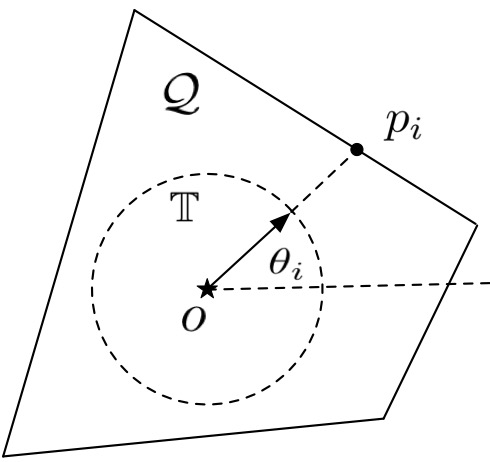}
\caption{The mapping from convex boundary $\partial\mathcal{Q}$ to unit circle $\mathbb{T}$.\label{fig:Map}}
\end{figure}

We assume that all robots follow simple first-order continuous-time motion model. Each robot $i$ knows its own position exactly at all times. When two robots communicate they can exchange their respective positions. \rev{We also assume that all robots have sensors that cover the environment, and can always communicate with their neighbors (i.e., robot $i$ can communicate with $i+1$ and $i-1$). In Section~\ref{sec:limited}, we derive necessary and sufficient conditions of the sensing and communication range.}

Let $\omega_{\max}$ denote the common maximum angular speed\footnote{Strictly speaking, each robot has a maximum speed with which it can move on $\partial\mathcal{Q}$. In Appendix~\ref{app:omega}, we show how the maximum speed on $\partial\mathcal{Q}$ can be used to determine $\omega_{\max}$.} for all robots on the unit circle. Our results can be extended to the situation where each robot has its own maximum angular speed. 

Martinez and Bullo~\cite{martinez2006optimal} showed that the optimal configuration for the robots that can obtain distance measurements towards the target is a uniform configuration along the circle where each pair of neighboring robots is equally spaced around the target. That is, $\theta_{i+1}-\theta_{i}=2\pi/N, \forall i\in \{1,...,N\}$. Optimality is defined with respect to maximizing the determinant of the Fisher Information Matrix (FIM). FIM is a commonly used measure for active target tracking. Martinez and Bullo~\cite{martinez2006optimal} presented a decentralized control law that is guaranteed to (asymptotically) converge to a uniform configuration when a robot is allowed to communicate with only two of its immediate neighbors. That is, a robot $i$ can communicate with only $i-1$ and $i+1$, along the circle. The analysis requires that all robots know the position of the target exactly and that the target remains stationary. In the same paper, they showed how to apply the same control law in situations where the target's position is not known exactly and is instead estimated by combining noisy range measurements in an Extended Kalman Filter. They also evaluated the performance of the algorithm empirically in cases where the target is allowed to move.

The control law in~\cite{martinez2006optimal} assumes that neighboring robots communicate at every time step. We call this the \emph{constant strategy}. Our objective in this work is to reduce the number of communications between the robots while still maintaining the convergence properties. We present a \emph{self-triggered strategy} where the control law for each robot not only decides how a robot should move, but also when it should  communicate with its neighbors and seek new information. We show that the proposed self-triggered strategy is also guaranteed to converge to a uniform configuration, under the model and assumptions described in this section. 

%
%

\section{Self-Triggered Tracking Algorithm}	\label{sec:self}
In this section, we present the self-triggered tracking algorithm for achieving a uniform configuration along the unit circle. This requires knowing the center of the circle (\ie, the target's position) and assuming that this center does not change. These assumptions are required for the convergence analysis to hold. We later relax these assumptions and present a practical version in the following section.

Our algorithm builds on the \texttt{self-triggered}\\ \texttt{centroid algorithm}~\cite{nowzari2012self} which is a decentralized control law that achieves optimal deployment (\ie, uniform Voronoi partitions) in a convex environment. We suitably modify this algorithm for the cases where the robots are restricted to move only on the boundary, $\partial\mathcal{Q}$, and can communicate with only two neighbors as described in the previous section. We first present the control law for each of the robots that uses the motion prediction set of its neighbors based on their last known positions. Then, we present an update policy to decide when a robot should communicate and seek new information from its neighbors. 

\subsection{Control Law}
The constant control law in~\cite{martinez2006optimal} drives every robot towards the midpoint of its Voronoi segment. The Voronoi segment of the robot $i$ is the part of the unit circle extending from $(\theta_{i-1}+\theta_{i})/2$ to $(\theta_{i}+\theta_{i+1})/2$. The constant control law steers robot $i$ towards the midpoint of its Voronoi segment $V_{\mid}^{i}$ by using real-time (at every time step\footnote{\rev{Denote one time step as a small time interval, $\Delta t$.}}) information from its neighbors, $\theta_{i-1}$ and $\theta_{i+1}$, as illustrated in Figure~\ref{fig:Cer}. We refer to the book \cite{okabe2009spatial} for a comprehensive treatment on Voronoi segment.

\begin{figure}
\centering
\includegraphics[width=0.65\columnwidth]{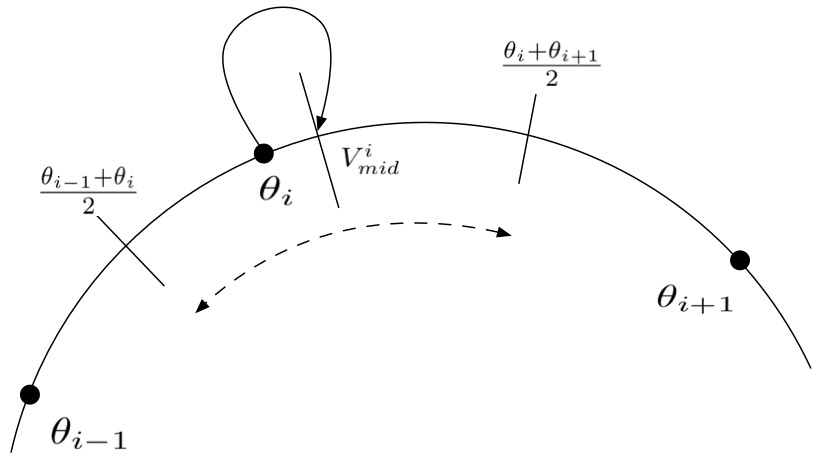}
\caption{Robot $i$ goes towards the midpoint of its Voronoi segment via exact information from its neighbors.\label{fig:Cer}}
\end{figure}


In distributed self-triggered strategies, exact positions of the neighbors is not always available in real-time. Consequently, the algorithm must be able to operate on this inexact information. The information that each robot $i$ holds about its neighbor $j$ is the last known position of $j$, denoted by $\theta_{j}^{i}$, and the time elapsed since the position of robot $j$ was collected, denoted by $\tau_{j}^{i}$. Based on this, robot $i$ can compute the furthest distance that $j$ could have moved in $\tau_{j}^{i}$ time as $\pm\phi_{j}^{i}$ where,
\begin{equation}
\phi_j^i =\omega_{\max} \tau_{j}^{i}.
\end{equation}
Thus, robot $i$ can build a prediction motion set $\mathcal{R}_{j}^{i}(\theta_{j}^{i}, \phi_{j}^{i})$ that  contains all the possible locations where robot $j$ could have moved to in $\tau_{j}^{i}$ time (Figure~\ref{fig:Rij}).

\begin{figure}
\centering
\includegraphics[width=0.5\columnwidth]{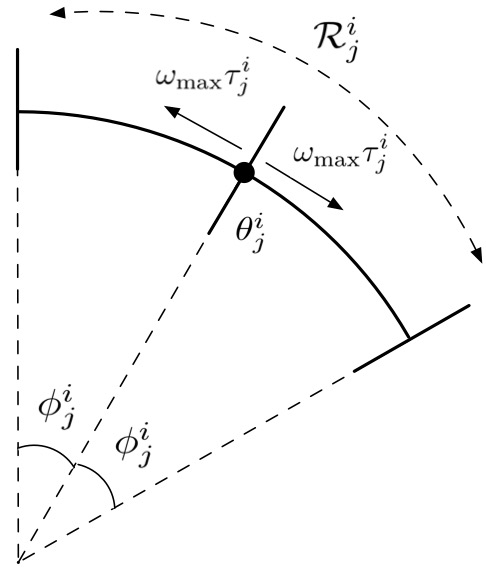}
\caption{Motion prediction set, $\mathcal{R}_j^i$, that each robot $i$ maintains for its neighbors $j$. $\theta_j^i$ is the last known position of robot $j$ and $\tau_j^i$ is the time elapsed since this last known position.\label{fig:Rij}}
\end{figure}


\rev{In our algorithm, it is sufficient for robot $i$ to only communicate with its neighbors $i-1$ and $i+1$.} The prediction motion range that robot $i$ stores is given as 
$\mathcal{R}^{i}:=\{\mathcal{R}_{i-1}^{i}(\theta_{i-1}^{i}, \phi_{i-1}^{i}), \mathcal{R}_{i+1}^{i}(\theta_{i+1}^{i}, \phi_{i+1}^{i})\}$. 
%
%





The proposed self-triggered strategy uses these motion prediction ranges $\mathcal{R}^{i}$ for defining the control law of robot $i$. Since the robot has inexact information of its neighbors, the midpoint of its Voronoi segment is a set instead of a point (Figure~\ref{fig:Unc}).

\begin{figure}
\centering
\includegraphics[width=0.9\columnwidth]{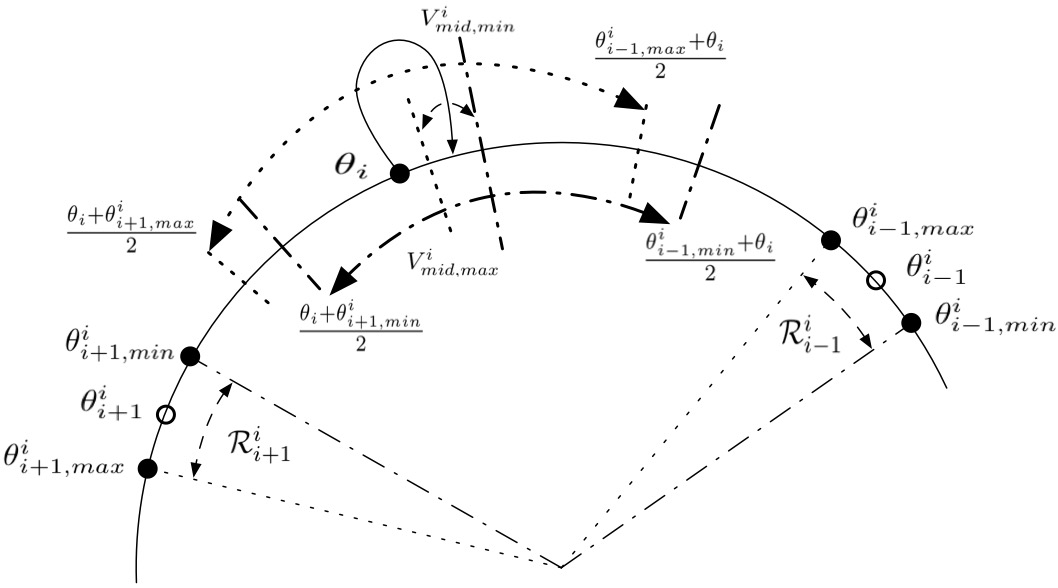}
\caption{Robot $i$ goes towards the midpoint of its Voronoi segment via inexact motion prediction about its neighbors.\label{fig:Unc}}
\end{figure}


Define:
\begin{align*}
\theta_{i-1,\min}^{i}&=\left(\theta_{i-1}^{i}-\phi_{i-1}^{i}\right)
&\theta_{i-1,\max}^{i}=\left(\theta_{i-1}^{i}+\phi_{i-1}^{i}\right)\\
\theta_{i+1, \min}^{i}&=\left(\theta_{i+1}^{i}-\phi_{i+1}^{i}\right)
&\theta_{i+1, \max}^{i}=\left(\theta_{i+1}^{i}+\phi_{i+1}^{i}\right).
\end{align*} 
Thus, we have:
\begin{align*}
\mathcal{R}_{i-1}^{i}(\theta_{i-1}^{i},  \phi_{i-1}^{i}) &=\{\beta\in \mathbb{T} |\theta_{i-1,\min}^{i}\leq \beta \leq \theta_{i-1,\max}^{i}\},\\
\mathcal{R}_{i+1}^{i}(\theta_{i+1}^{i}, \phi_{i+1}^{i})&=\{\beta\in \mathbb{T} |\theta_{i+1, \min}^{i} \leq \beta \leq \theta_{i+1, \max}^{i}\}.
\end{align*}

Then the minimum and maximum midpoints of robot $i$'s Voronoi segment can be computed as,
\begin{align}
V_{\mid,\min}^{i}&=\frac{(\theta_{i-1,\min}^{i}+\theta_{i})/2+(\theta_{i}+\theta_{i+1,\min}^{i})/2}{2},\label{eqn:vmidmin}\\
V_{\mid,\max}^{i}&=\frac{(\theta_{i-1,\max}^{i}+\theta_{i})/2+(\theta_{i}+\theta_{i+1,\max}^{i})/2}{2}.\label{eqn:vmidmax}
\end{align}

The midpoint of its Voronoi segment

\noindent $V_{\mid}^{i}\in[V_{\mid,\min}^{i},V_{\mid,\max}^{i}]$. That is,
\begin{equation}
V_{\mid,\min}^{i} \leq V_{\mid}^{i} \leq V_{\mid,\max}^{i}.\label{eqn:vmid}
\end{equation}

\noindent Substitute Equations~\ref{eqn:vmidmin} and \ref{eqn:vmidmax} into Equation~\ref{eqn:vmid} yields,
\begin{equation*}
\frac{\theta_{i+1}^{i}+2\theta_{i}+\theta_{i-1}^{i}-2\omega_{\max}\tau^{i}}{4} \leq V_{\mid}^{i}
\end{equation*} and 
\begin{equation*}
V_{\mid}^{i} \leq \frac{\theta_{i+1}^{i}+2\theta_{i}+\theta_{i-1}^{i}+2\omega_{\max}\tau^{i}}{4},
\end{equation*} 
then \begin{equation*}
-\frac{\omega_{\max}\tau^{i}}{2}\leq V_{\mid}^{i}-\frac{\theta_{i+1}^{i}+2\theta_{i}+\theta_{i-1}^{i}}{4} \leq\frac{\omega_{\max}\tau^{i}}{2}.
\end{equation*}
Therefore,
\begin{equation}
\left|V_{\mid}^{i}-\frac{\theta_{i+1}^{i}+2\theta_{i}+\theta_{i-1}^{i}}{4}\right| \leq  \frac{\omega_{\max}\tau^{i}}{2}.\label{eqn:ubd}
\end{equation}
 
Thus, the angular distance between $V_{\mid}^{i}$ and

\noindent$\frac{\theta_{i+1}^{i}+2\theta_{i}+\theta_{i-1}^{i}}{4}$ is bounded by $\frac{\omega_{\max}\tau^{i}}{2}$. In fact, the point $\frac{\theta_{i+1}^{i}+2\theta_{i}+\theta_{i-1}^{i}}{4}$ indicates the midpoint of $i$'s guaranteed Voronoi segment $gVs_{i}$, defined as,
\begin{equation*}
gVs_{i}=\left\lbrace\beta\in \mathbb{T} \left|\max_{\theta_{i}\in S_{i}}\left|\beta-\theta_{i}\right| \leq \min_{\theta_{j}\in S_{j}}\left|\beta-\theta_{j}\right|,\forall  j\neq i\right.\right\rbrace
\end{equation*}
where $T_{1},\ldots,T_{n} \subset \mathbb{T}$ are a set of connected segments in $\mathbb{T}$. We refer to the report\cite{evans2008guaranteed} for more details on the guaranteed Voronoi segment. Thus, the guaranteed Voronoi segment of robot $i$ can be computed as,
\begin{equation}
gVs_{i}=\left\lbrace\beta \left| \frac{\theta_{i}+\theta_{i+1,\min}^{i}}{2} \leq \beta \leq \frac{\theta_{i-1,\max}^{i}+\theta_{i}}{2}\right.\right\rbrace.
\end{equation}
Although robot $i$ does not know the exact midpoint of its Voronoi segment $V_{\mid}^{i}$, it can move towards the midpoint of its guaranteed Voronoi segment $gV_{\mid}^{i}$ instead, which is given by,
\begin{eqnarray}
gV_{\mid}^{i}&=&\frac{(\theta_{i}+\theta_{i+1,\min}^{i})/2+(\theta_{i-1,\max}^{i}+\theta_{i})/2}{2},\nonumber\\
&=& \frac{\theta_{i+1}^{i}+2\theta_{i}+\theta_{i-1}^{i}}{4}.
\end{eqnarray}

In general, moving towards $gV_{\mid}^i$ does not guarantee that the robot moves closer to the midpoint of its Voronoi segment. However, the statement holds under the following condition.
\begin{lemma}
Suppose robot $i$ moves from $\theta_i$ towards $gV_{\mid}^i$. Let $\theta_i'$ be its position after one time step. If $|\theta_i' - gV_{\mid}^i| \geq |V_{\mid}^i - gV_{\mid}^i|$, then $|\theta_i'-V_{\mid}'| \leq |\theta_i-V_{\mid}^i|$.
\end{lemma}

The proof for this lemma follows directly from the proof for Lemma 5.1 in~\cite{nowzari2012self}. Consequently, as long as the robot can ensure that its new position $\theta_i'$ satisfies $|\theta_i' - gV_{\mid}^i| \geq |V_{\mid}^i - gV_{\mid}^i|$, then it is assured to not increase its distance from the actual (unknown) midpoint of the Voronoi segment. However, the right-hand side of this condition also is not known exactly since robot $i$ does not know $V_{\mid}^i$. Instead, we can set an upper bound on this term using Equation~\ref{eqn:ubd}. We denote this upper bound by $\mathrm{ubd}_{i}:=\frac{\omega_{\max}\tau^{i}}{2}$. Thus, we get the following result:

\begin{corollary}

Suppose robot $i$ moves from $\theta_i$ towards $gV_{\mid}^i$. Let $\theta_i'$ be its position after one time step. If 
\begin{equation}
|\theta_i' - gV_{\mid}^i| > \mathrm{ubd}_{i},
\label{eqn:ubdcondition}
\end{equation}

then $|\theta_i'-V_{\mid}'| \leq |\theta_i-V_{\mid}^i|$.
\label{cor:ubd}
\end{corollary}

Next, we present a motion control law that steers the robots towards a uniform configuration on the circle. Intuitively, robot $i$ computes its guaranteed Voronoi segment (7) using the last known positions of its neighbors and the real-time position of itself. It then computes the midpoint of its guaranteed Voronoi segment (8) and moves towards the midpoint until it is within $\mathrm{ubd}_{i}$ of it. Formally, the control, $u_{i}(t_{k})$, for robot $i$ at time $t_k$ is given by:
\begin{equation}
u_{i}(t_{k})=\omega_{i}\;\mathrm{unit}(gV_{\mid}^{i}-\theta_{i}),
\label{eqn:controllaw}
\end{equation}
where,
$$\omega_{i}=
\begin{cases}
\omega_{\max},& \left|gV_{\mid}^{i}-\theta_{i}\right|\geq \mathrm{ubd}_{i}+\omega_{\max}\Delta t,\\
0,& \left|gV_{\mid}^{i}-\theta_{i}\right|\leq \mathrm{ubd}_{i},\\
\frac{\left|gV_{\mid}^{i}-\theta_{i}\right|- \mathrm{ubd}_{i}}{\Delta t},&  \mathrm{otherwise}.
\end{cases}$$

%
%
%
%
%
%
%
%
%
%
%
%

\subsection{Triggering Policy}
As time elapses, without new information the upper bound $\mathrm{ubd}_{i}$ grows larger until the condition in Equation~\ref{eqn:ubdcondition} is not met. This triggers the robot to collect the updated information from its neighbors. There are two causes that may lead to the condition in Equation~\ref{eqn:ubdcondition} being violated. The upper bound on the right-hand side, $\mathrm{ubd}_{i}$, might grow large because of the time elapsed since the last communication occurred. Or, robot $i$ might move close to $gV_{\mid}^{i}$ which would require $\mathrm{ubd}_{i}$ to become small by acquiring new information. The second scenario might lead to frequent triggering when the robots are close to convergence. We introduce a user-defined tolerance parameter, $\sigma\geq 0$,  to relax the triggering condition. Whenever the following condition is violated, the robot is required to trigger new communication:
\begin{equation}
\mathrm{ubd}_{i} < \max\{\|\theta'-gV_{\mid}^{i}\|, \sigma \}
\label{eqn:ubdconditiontolerance}
\end{equation}

Furthermore, the motion control law is designed under the assumption that the robot $i$ and its two neighbors are located in the counterclockwise order. That is, $\theta_{i+1}> \theta_{i} > \theta_{i-1}$. Since the robots are identical, it is clear that there is no advantage gained by changing the order of robots along the circle. In a constant strategy, since the robots always communicate, they know the real-time position of their neighbors and can thus avoid the order being swapped. In a self-triggered strategy, however, we only have a motion prediction set of the neighbors. If there is a possibility that this order may be violated, the robots must communicate and avoid it. We achieve this by requiring the robot to maintain the following condition:
\begin{equation}
\theta_{i+1}^{i}-\omega_{\max}\tau_{i+1}^{i}> \theta_{i} > \theta_{i-1}^{i}+\omega_{\max}\tau_{i-1}^{i}
\label{eqn:thetacond}
\end{equation}

This ensures that even in the worst case, the robots have not swapped their positions. Whenever there is a possibility of this condition being violated, the robot $i$ triggers a new communication.
%
%
%
%
%
%
%
%
%
%
%
%
%
%
%
%
%
%
%
%

The complete self-triggered midpoint strategy is presented below:

\hangafter 1
\hangindent 1.5em

\rule{0.9\columnwidth}{1.5pt}

$\textbf{Algorithm 1:}$ \textsc{Self-triggered Midpoint}

\rule{0.9\columnwidth}{1pt}

~1: \textbf{while} all robots have not converged:

~2: ~~~\textbf{for} each robot $i\in\{1,...,N\}$ perform: 

~3: ~~~~~~increment $\tau_{i-1}^i$ and $\tau_{i+1}^i$ by $\Delta t$

~4: ~~~~~~compute $\mathcal{R}^{i}, gVs_{i}, gV_{\mid}^{i},$ and $\mathrm{ubd}_{i}$



~5: ~~~~~~\textbf{if} Equation~\ref{eqn:ubdconditiontolerance} OR Equation~\ref{eqn:thetacond} is violated:

~6: ~~~~~~~~~trigger communication with $i+1$ and $i-1$

~7: ~~~~~~~~~reset $\tau_{i+1}^i$ and $\tau_{i-1}^i$ to zero

~8: ~~~~~~~~~recompute $\mathcal{R}^{i}, gVs_{i}, gV_{\mid}^{i},$ and $\mathrm{ubd}_{i}$

~9: ~~~~~~\textbf{end if}

10: ~~~~~~compute and apply $u_{i}$ as defined in Equation~\ref{eqn:controllaw}

11: ~~~\textbf{end for}

12: \textbf{end while}

\rule{0.9\columnwidth}{1pt}

\subsection{Convergence Analysis} 

Algorithm 1 is guaranteed to converge asymptotically to a uniform configuration along the circumference of the circle, irrespective of the initial configuration, assuming that no two robots are co-located initially. The proof for the convergence follows directly from the proof of Proposition 6.1 in~\cite{nowzari2012self} with suitable modifications. In the following, we sketch these modifications.

In~\cite{nowzari2012self} the robots are allowed to move anywhere in the interior of $Q\subset \mathbb{R}^2$ whereas in our case the robots are restricted to move on $\partial Q$, equivalent to moving on the unit circle $\mathbb{T}$. Therefore, all the $L_2$ distances in the proof in \cite{nowzari2012self} change to $L_1$ distances. Instead of moving to the midpoint of the guaranteed Voronoi segment, the robots in~\cite{nowzari2012self} move to the centroid of a guaranteed Voronoi region. Instead of communicating with the two clockwise and counter-clockwise neighbors, the robots in~\cite{nowzari2012self} communicate with all possible Voronoi neighbors. None of these changes affect the correctness of the proof. We add an extra condition that triggers communications to prevent robots from changing their order along $\mathbb{T}$. Since this condition only results in additional triggers, it can only help convergence. Finally, since there is a one-to-one and onto mapping between $\partial Q$ and $\mathbb{T}$, convergence along $\mathbb{T}$ implies convergence along $\partial Q$.

\section{Practical Extensions} \label{sec:practical}
\rev{In this section, we present two practical extensions of our algorithm relaxing some of the assumptions given in Section~\ref{sec:probform}.}
\subsection{Tracking of Moving Target with Noisy Measurements} 
If the true position of the target, $o^*$, is known, then we can draw a unit circle centered at the target and use the strategy in Algorithm 1 to converge to a uniform configuration along the circle. According to the result in~\cite{martinez2006optimal} this configuration maximizes the determinant of the FIM. In practice, however, we do not know the true position of the target. In fact, the goal is to use the noisy measurements from the robots to estimate the position of the target. Furthermore, the target may be mobile. This implies that the (unknown) center of the circle is also moving, further complicating the control strategy for the robots.

We use an Extended Kalman Filter (EKF) that estimates the position of the target (\ie, center) and predicts its motion at every time step. The prediction and the estimate of the target from an EKF is a 2D Gaussian distribution parameterized by its mean, $\hat{o}(k)$ and covariance $\hat{\Sigma}(k)$. \rev{The target's state prediction and update by EKF are described below.

\noindent\emph{Prediction:}
$$\hat{o}^{-}(k)= \hat{o}(k-1),$$
$$\hat{\Sigma}^{-}(k) = \hat{\Sigma}(k-1)+R(k).$$

\noindent\emph{Update:}
$$K(k) = \hat{\Sigma}^{-}(k) H^{T}(k)(H(k)\hat{\Sigma}^{-}(k)H^{T}(k) + Q(k))^{-1},$$
$$\hat{o}(k) = \hat{o}^{-}(k) + K(k)(z(k)-h(\hat{o}^{-}(k)))$$
$$\hat{\Sigma}(k)= (I-K(k)H(k))\hat{\Sigma}(k)^{-}$$
where $R(k)$ and $Q(k)$ are the covariance matrices of the noise from target's motion model and robot's measurement, respectively. $h(\hat{o}^{-}(k)):=\|p(k)-\hat{o}^{-}(k)\|_2$. $z(k)$ denotes the noisy distance measurement from the robot. $H(k)$ is the Jacobean of $h(\hat{o}^-(k))$.} At each time step, we use the mean of the latest estimate as the center of the circle to compute the $\theta_i$ values using the transformation in Equation~\ref{eqn:transformation}. 

In the centralized setting, a common fusion center can obtain the measurements from all the robots and compute a single target estimate, $\hat{o}(k)$ at every time step. Therefore, each robot will have the same estimated mean, $\hat{o}(k)$, and therefore the same center for the unit circle. However, in the decentralized case, each robot runs its own EKF estimator and has its own target estimate, $\hat{o}^i(k)$, based on only its own measurements of the target. As a result, the centers of the unit circle will not be the same, making convergence challenging. 

If at any time step, a robot communicates with its neighbors, then it can also share its estimate (mean $\hat{o}(k)$ and covariance $\hat{\Sigma}(k)$) with its neighbors. Therefore, at these triggered instances, each robot can update its own estimate by fusing the estimates from its neighbors. We use the covariance intersection algorithm, which is a standard decentralized EKF technique, to fuse estimates under unknown corrections~\cite{reinhardt2012closed}. 

The covariance intersection algorithm takes two Gaussian beliefs, ($x_a,\Sigma_a$) and ($x_b,\Sigma_b$), and combines them into a common belief, ($x_c,\Sigma_c$):
\begin{eqnarray*}
&&x_c = \Sigma_c((\Sigma_a)^{-1}x_a+(\Sigma_b)^{-1})^{-1}x_b)\\
&&\Sigma_c = (\lambda(\Sigma_a)^{-1} + (1-\lambda)(\Sigma_b)^{-1})^{-1}
\label{eqn:covariace_intersection}
\end{eqnarray*}
Here, $\lambda\in[0,1]$ is a  design parameter obtained by optimizing some criteria, i.e.,  determinant or trace of $\Sigma_c$.

The rest of the process is similar to that in Algorithm 1. The centralized EKF scheme is a baseline which we compare against for the more realistic decentralized strategy. The results are presented in the simulation section that follows.

\rev{
\subsection{Limited Communication and Sensing Range} \label{sec:limited}
Our main result assumes that the robots have sufficiently large communication and sensing ranges. In this section, we first present a necessary condition for the communication range $r_c$ and sensing range $r_s$. We then present a sufficient condition on the communication range for a modified version of our algorithm.

\begin{theorem}[Necessary Condition]
Let $N$ be the total number of robots. To guarantee the convergence to the optimal configuration when the robots do not know $N$, the communication range $r_c$  cannot be less than $ D_{\text{in}}\sin\frac{\pi}{N}$ and the sensing range $r_s$ cannot be less than $\frac{D_{\text{in}}}{2}$. $D_{\text{in}}$ indicates the diameter of the largest radius circle contained completely inside the environment.
\label{thm:nessary_cond}
\end{theorem}
\begin{proof}
\begin{figure*}[htb]
\centering{
\subfigure[]{\includegraphics[width=0.5\columnwidth]{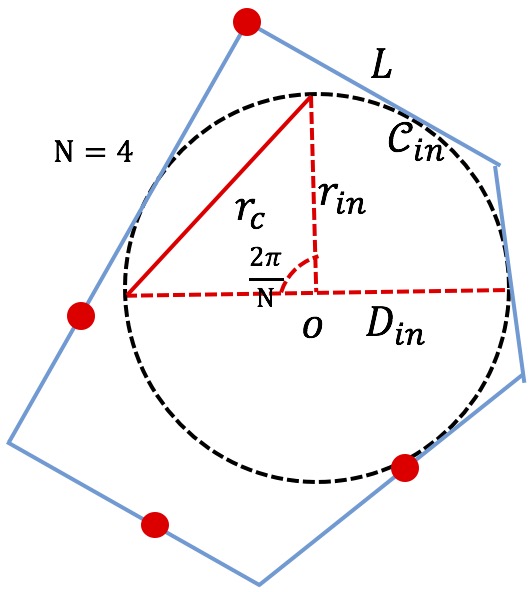}}
\subfigure[]{\includegraphics[width=0.5\columnwidth]{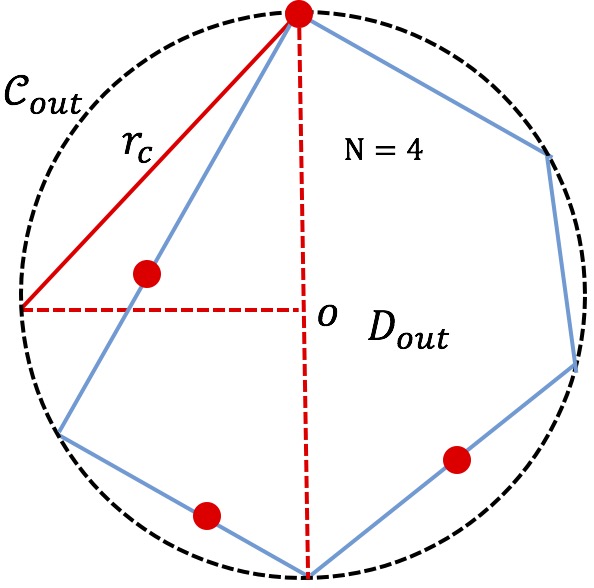}}
\subfigure[]{\includegraphics[width=0.5\columnwidth]{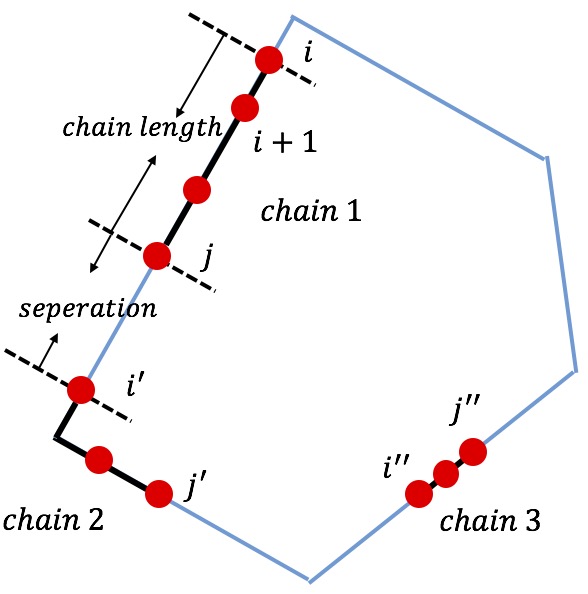}}
\caption{The example of four robots moving on an arbitrary convex boundary to show necessary condition: (a) and sufficient condition: (b) and (c). The red solid circle indicates the robot. 
\label{fig:necessary_sufficient}}}
\end{figure*}

Consider an arbitrary convex boundary as shown in Figure~\ref{fig:necessary_sufficient}-(a). We draw its inscribed circle $\mathcal{C}_{\text{in}}$ with radius $r_{\text{in}}$ and diameter $D_{\text{in}}$. To guarantee convergence without knowing $N$, the robots must be able to communicate with both neighbors when they reach a uniform configuration. When $N=4$ (Figure~\ref{fig:necessary_sufficient}-(a)), if the communication range among any two robots, $r_c < \sqrt{2}r_{\text{in}} = \frac{\sqrt{2}D_{\text{in}}}{2}$, these four robots cannot communicate with each other even when they are at the uniform configuration. For any $N$, $r_c$ can be calculated by using the cosine law,
$$r_{c}^{2} = 2r_{\text{in}}^{2}-2r_{\text{in}}^{2}\cos\frac{2\pi}{N}.$$ Thus, $$r_{c} = 2r_{\text{in}}\sin\frac{\pi}{N} = D_{\text{in}}\sin\frac{\pi}{N}.$$ Thus, we obtain the necessary condition that $r_{c}$ cannot be less than $D_{\text{in}}\sin\frac{\pi}{N}$. 

If $r_s < r_{\text{in}}$, no robot can sense the target when the target is at the center of the circle. Thus the sensing range $r_s$ cannot be less than $\frac{D_{\text{in}}}{2}$.

\label{pro:necessary_cond}
\end{proof}

Next, we propose a sufficient condition on the communication and sensing ranges to ensure convergence to the uniform configuration. We need to make an additional assumption that each robot can uniquely identify its forward and backward neighbors. We also assume that the communication range is the same for all the robots and is known to all the robots. We present a modified version of our strategy that works with limited communication range.

\noindent\emph{Modified Self-Triggered Strategy: If robot $i$ cannot communicate with either of its two neighbors, it does not move. If robot $i$ can only communicate with one of its neighbors, it moves in the direction of the other neighbor with maximum velocity. A robot keeps moving unless its motion will cause it to lose communication with its neighbors. If robot $i$ can communicate with both of its neighbors, it applies the proposed control law (Equation~\ref{eqn:controllaw}).}

\begin{theorem}[Sufficient Condition]
If the communication range $r_c \geq \frac{L}{N}$ and the sensing range $r_s \geq {D_{\text{out}}}$, then the modified strategy converges to the optimal configuration. Here, $L$ and $D_{\text{out}}$ indicate the environment's perimeter and the length of the longest segment contained completely inside the environment, respectively. 
\label{thm:sufficient_cond}
\end{theorem}
\begin{proof}
We define a \emph{communication chain} (Figure~\ref{fig:necessary_sufficient}-(c)) to be the maximal set of consecutive robots, $i,i+1,\ldots,j$, such that $i$ can communicate with $i+1$, $i+1$ can communicate with $i+2$, and so on until $j$. We now show that irrespective of the starting configuration, using the modified control law, all the $N$ robots will form a single chain. 

We define \emph{length} of a chain to be the distance along the boundary (in the direction that contains the chain) between the two extreme robots in a chain (Figure~\ref{fig:necessary_sufficient}-(c)). We denote two extreme robots as the robots at the two endpoints of the chain. 

Consider a chain of $K$ robots. We show that the robots in this chain will keep moving unless the length is greater than or equal to $\frac{L}{N}(K-1)$ or the chain merges with another. Extreme robots in a chain have only one neighbor that they can communicate with. According to the strategy, these robots will continuously move (in a direction away from the chain) with maximum velocity. Other robots between the two extreme robots in the communication chain apply self-triggered control law (Equation~\ref{eqn:controllaw}) to go towards the midpoint of its two neighbors. Therefore, the length of the chain keeps increasing as long as the robots are moving. Unless the chain merges with another one, the robots will stop moving when the distance between all consecutive pairs of robots is $r_c$. Here, $r_c \geq \frac{L}{N}$. If two consecutive robots are on the same environment edge, then the distance along the boundary between the robots is exactly equal to $r_c$. If the two robots are on different boundary edges, then the distance between the robots along the boundary will be greater than $r_c$ (due to the convexity of the environment). Therefore, the length of the chain when all $K$ robots stop moving will be greater than or equal to $\frac{L}{N}(K-1)$.

Next, we prove our claim that eventually all robots form a single chain, by contradiction. Denote the \emph{separation} between two consecutive chains as the distance between the starting (ending) robot of one chain and the ending (starting) robot of another chain along the boundary of the environment (Figure~\ref{fig:necessary_sufficient}-(c)). Suppose, for contradiction, that there exists $M>1$ chains after all robots have stopped moving. Let $K_1,\ldots,K_M$ be the number of robots in the $M$ chains. $K_1+\cdots+K_M=N$. 

The separation between any two consecutive chains is strictly greater than $\frac{L}{N}$. Furthermore, the length of any chain is greater than or equal to $\frac{L}{N}(K_i-1)$. The perimeter of the environment must be equal to the length of all chains and the separation between all consecutive chains. Therefore, the perimeter must be strictly greater than
$$\frac{L}{N}(K_1-1) + \cdots + \frac{L}{N}(K_M-1) +  M \frac{L}{N} = L.$$
This contradicts with the fact that the perimeter of the environment is exactly $L$. Thus, we prove all the robots eventually form a single chain. 

Finally, once we ensure that robots form a single chain, then the convergence proof follows from the convergence of the self-triggered policy.

\label{pro:sufficient_cond}
\end{proof}
}

\section{Simulation and Outdoor Experiment} \label{sec:sims}
In this section, we evaluate the performance of the proposed self-triggered tracking coordination algorithm. We first compare the convergence time for the self-triggered and constant communication strategies to achieve a uniform configuration on a convex boundary (Section~\ref{sec:self}). Then, we demonstrate the performance of the self-triggered and constant strategies for moving targets.

\subsection{Stationary Target Case}

In this section, we compare the performance of the self-triggered and constant strategies in terms of their convergence speeds and the number of communication messages to achieve a uniform configuration on the boundary of a convex environment. Here, we focus on the base case of known, stationary target position. All results are for $30$ trials where the initial positions of the robots are drawn uniformly at random on the boundary. Our MATLAB implementation is also available online.\footnote{\url{https://github.com/raaslab/Self-triggered-mechanism}}

\begin{figure*}[htb]
\centering
\subfigure[$k=1$]{\includegraphics[width=0.67\columnwidth]{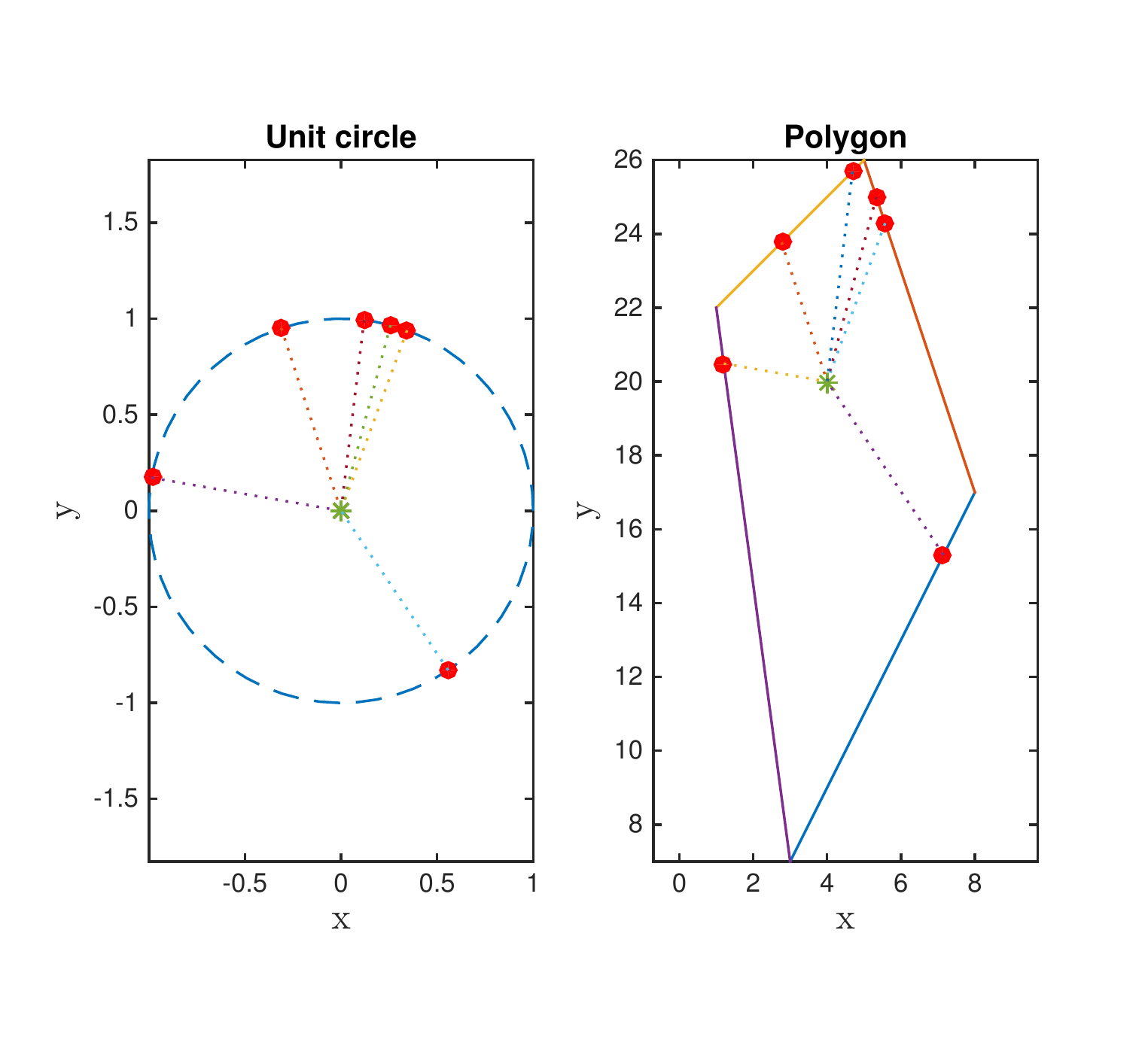}}
\subfigure[$k=400$]{\includegraphics[width=0.67\columnwidth]{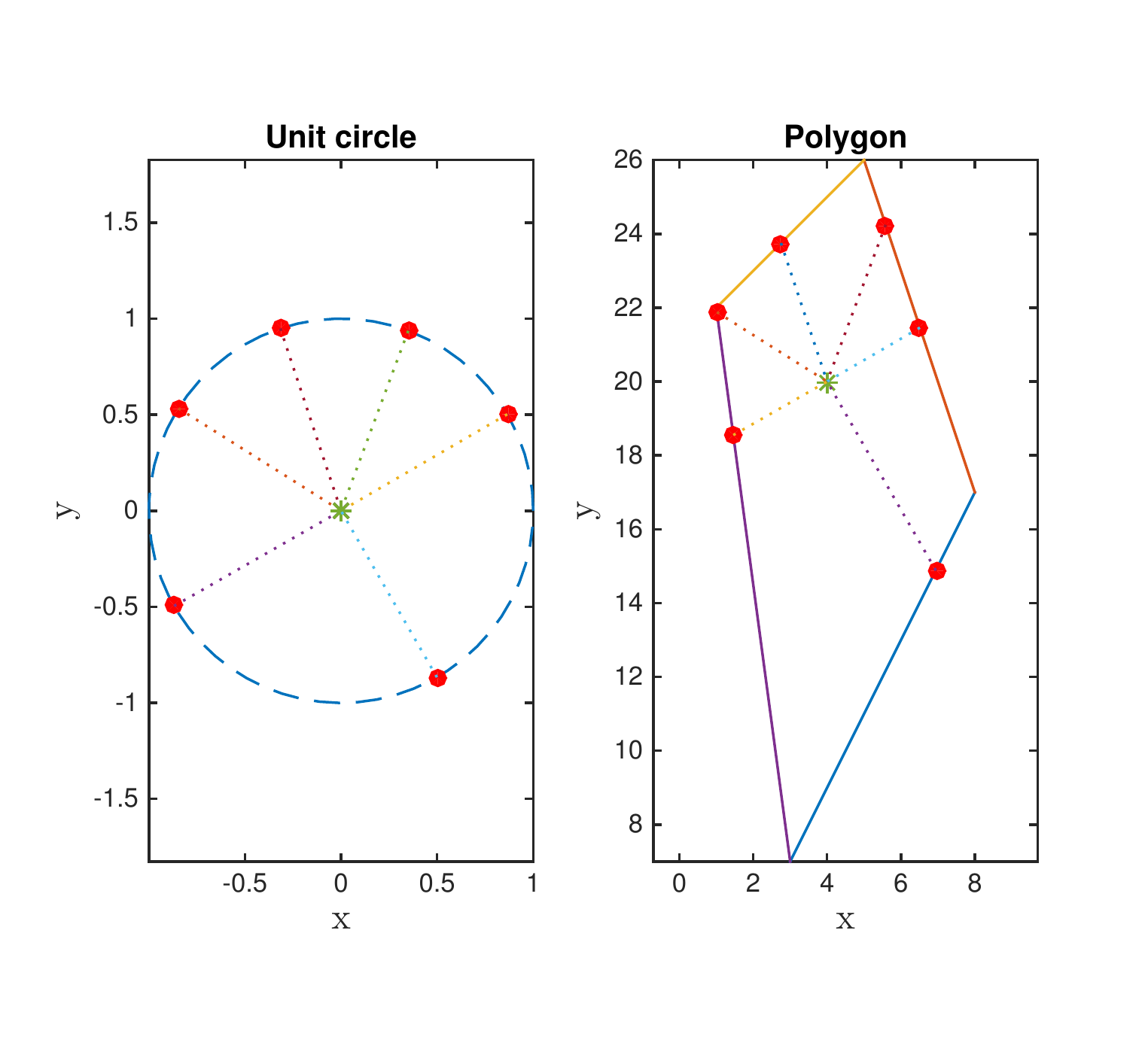}}
\subfigure[$k=760$]{\includegraphics[width=0.67\columnwidth]{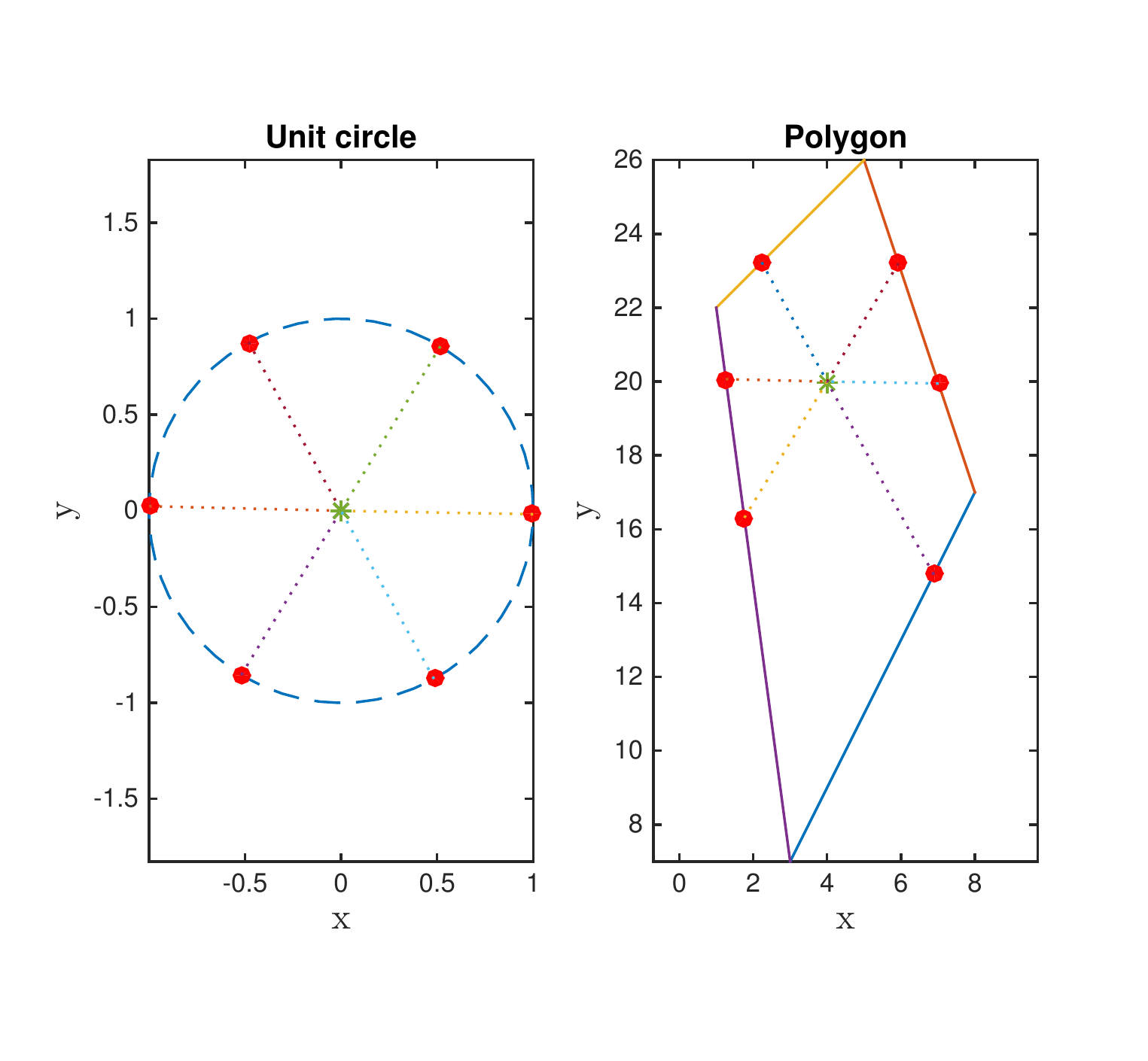}}
\caption{Self-triggered tracking with six robots moving on the boundary of a convex polygon with a known, stationary target. The robots took $760$ time steps to converge to the uniform configuration around the target. \label{fig:stationary}}
\end{figure*}

Figure~\ref{fig:stationary} shows snapshots of the active tracking process under the proposed self-triggered strategy starting with the initial configuration at time step $k=1$ in Figure~\ref{fig:stationary}-(a) and ending in a uniform configuration around the target at $k=760$ as shown in Figure~\ref{fig:stationary}-(c). For this example, we assume that the robots know the position of the stationary target. At each time step, we use the map $\varphi_{o}$ to find $\theta_i$ on the unit circle (Equation~\ref{eqn:transformation}), compute the control law as per Algorithm 1, and apply the inverse map $\varphi^{-1}$ to compute the new positions of the robots on $\partial Q$. We set $\Delta t=0.1\,s$ and assume that each robot has the same maximum angular velocity $\omega_{\max}=\frac{\pi}{180}\,\frac{rad}{s}$. In general, one can use the procedure given in the appendix to compute $\omega_{\max}$ for a given environment. \rev{Note that, the convergence time depends on $\omega_{\max}$, which in turn, depends on the shape of the environment assuming a fixed maximum linear velocity. In Figure~\ref{fig:ctime_omega}, we plot the convergence time for six robots starting from a fixed configuration by varying $\omega_{\max}$ from $\frac{\pi}{180}\frac{rad}{s}$ to $\frac{\pi}{2} \frac{rad}{s}$. It shows that the convergence time approaches a limit with increasing $\omega_{\max}$.} 

\begin{figure}
\centering
\includegraphics[width=0.65\columnwidth]{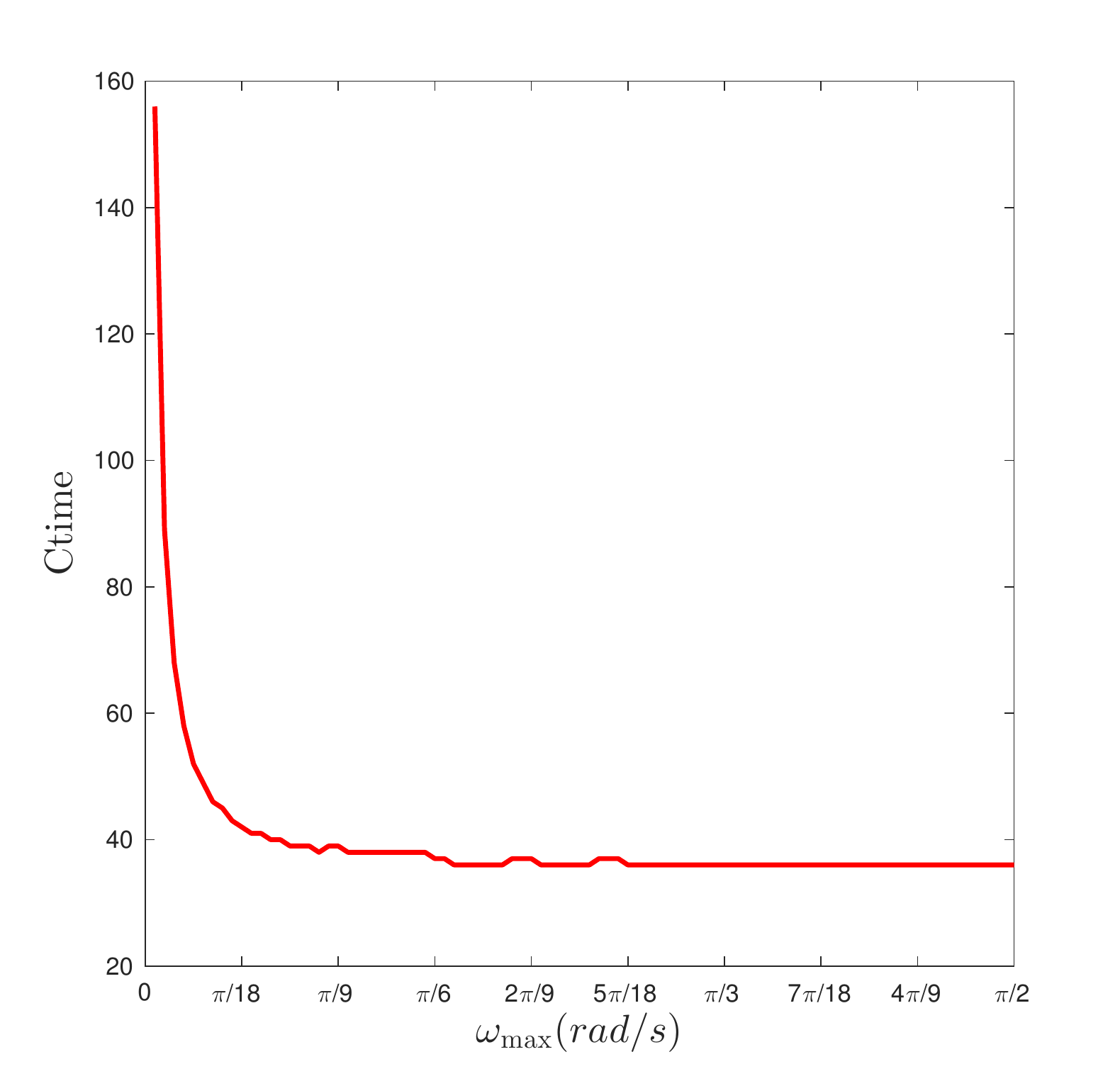}
\caption{The convergence time for six robots starting with same initial configuration for increasing values of maximum angular velocity $\omega_{\max}$.\label{fig:ctime_omega}}
\end{figure}

\begin{figure}[htb]
\centering{
\subfigure[]{\includegraphics[width=0.67\columnwidth]{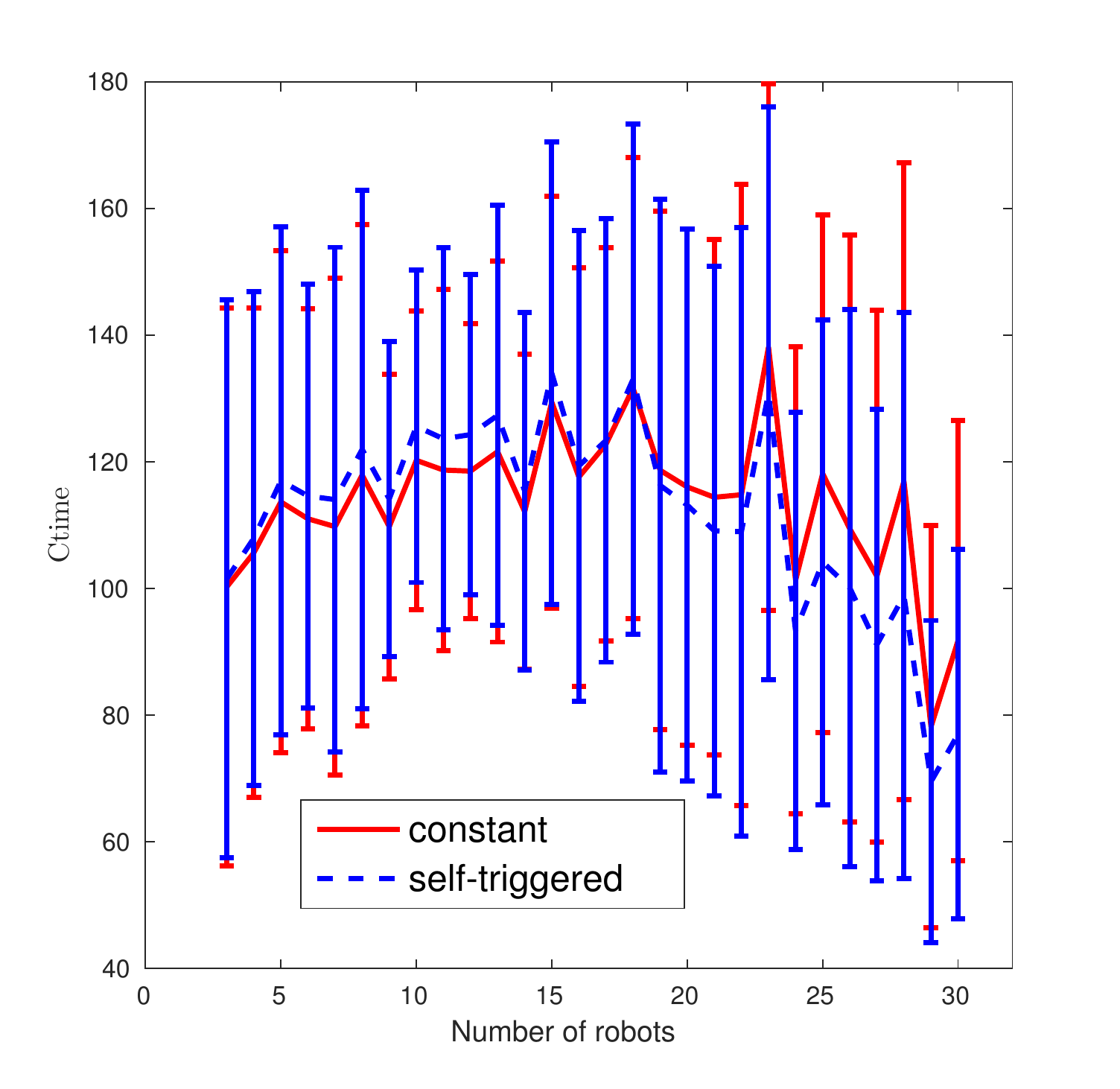}}
\subfigure[]{\includegraphics[width=0.67\columnwidth]{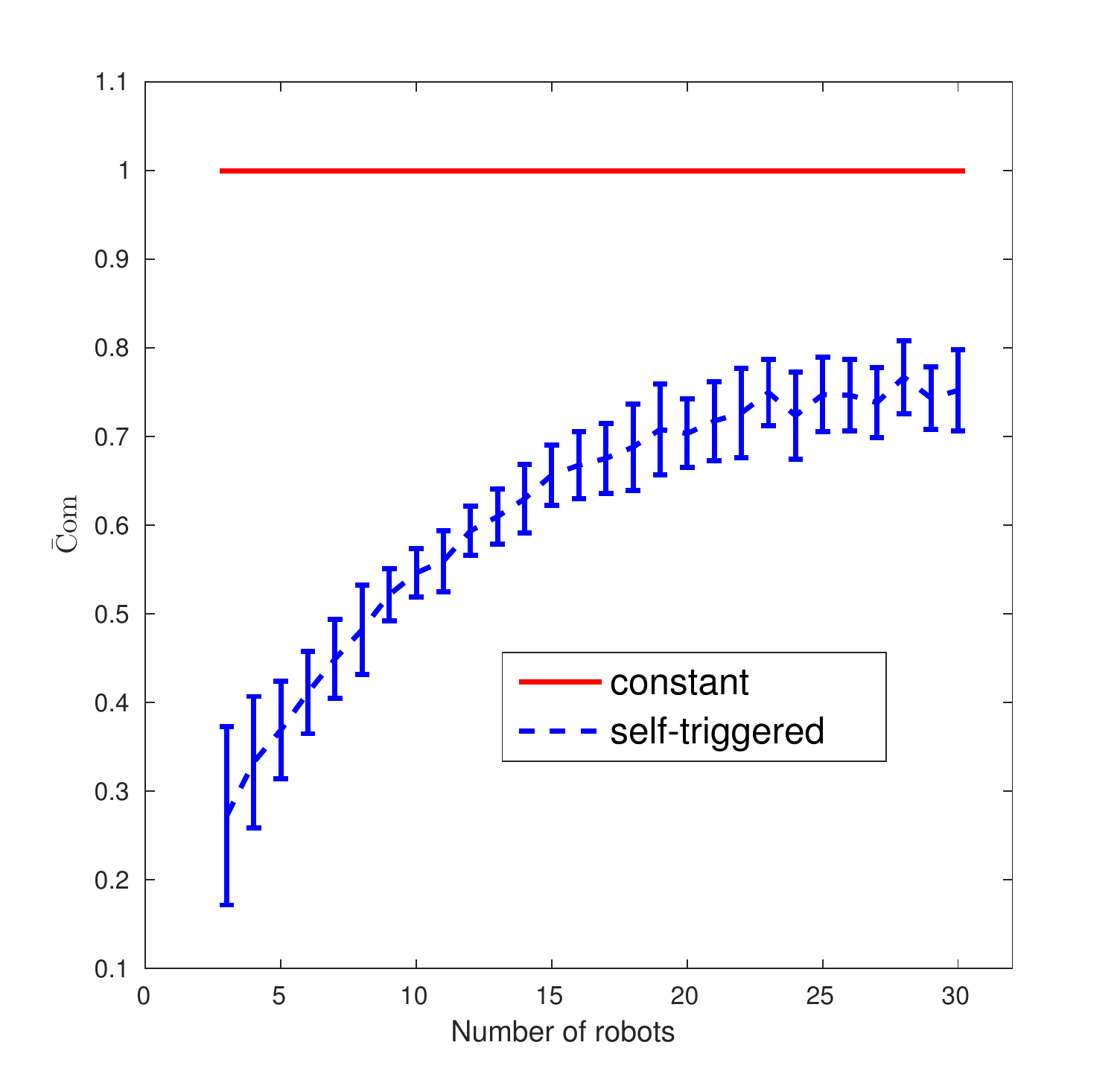}}
}
\caption{Comparison of the convergence time {(a)} and the number of communication messages {(b)} 
in constant and self-triggered strategies with a stationary target at known position. \rev{The error bar indicates standard deviation.} \label{fig:comp}}
\end{figure}


We first compare the convergence time of the two strategies \rev{with the same starting configurations for 30 trials} (Figure~\ref{fig:comp}-(a)). The convergence time, $\mathrm{Ctime}$ is specified as the timestep $k$ when the convergence error, $\mathrm{Cerr}$, drops below a threshold. We use $0.1N$ as the threshold, where $N$ is the number of robots. The convergence error term, $\mathrm{Cerr}$, is defined as:
\begin{equation}
\mathrm{Cerr}=\sum_{i=1}^{N} \left|\theta_{i}-V_{\mid}^{i}\right|
\label{eqn:conerrconstant}
\end{equation}
in the constant communication case, and
\begin{equation}
\mathrm{Cerr}=\sum_{i=1}^{N} \left|\theta_{i}-gV_{\mid}^{i}\right|
\label{eqn:conerrself}
\end{equation}
in the self-triggered case. 


The average number of communication messages is found as:
\begin{equation*}
\mathrm{\bar{Com}}=\frac{\sum_{i=1}^{N}\mathrm{com}(i,\mathrm{Ctime})}{N\times \mathrm{Ctime}}
\end{equation*}
where $\mathrm{com}(i,\mathrm{Ctime})$ gives the total number of communications of a robot with its neighbors $i$ at the end of $\mathrm{Ctime}$.  Figure~\ref{fig:comp}-(b) shows the $\mathrm{\bar{Com}}$ in the self-triggered case. The number of communication messages in the constant communication case is a constant. 
Figure~\ref{fig:comp}-(a) shows that the self-triggered mechanism converges comparatively with the constant strategy.

We also implemented our algorithm in ROS and performed simulations in the Gazebo environment~\cite{koenig2004design}. Figure~\ref{fig:gazebo} shows an instance with six differential-drive Pioneer 3DX robots~\cite{mei2005case} that can move in forwards and backwards direction. 

\begin{figure}[htb]
\centering
\includegraphics[width=0.65\columnwidth]{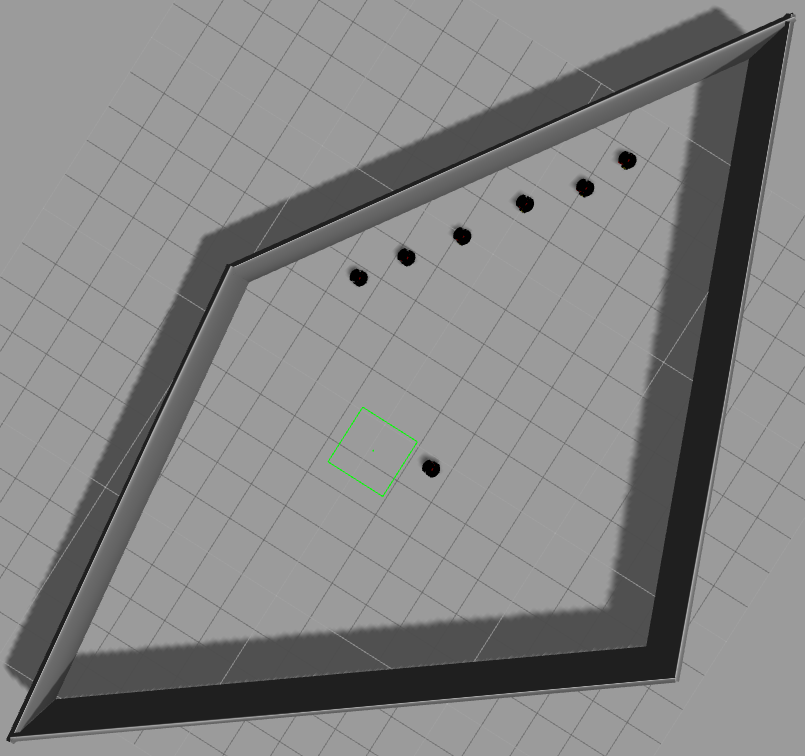}
\caption{Gazebo environment where six simulated Pioneer 3DX robots are tasked to track a target moving in the interior.\label{fig:gazebo}}
\end{figure}

\begin{figure}[htb]
\centering{
\subfigure[]{\includegraphics[width=0.67\columnwidth]{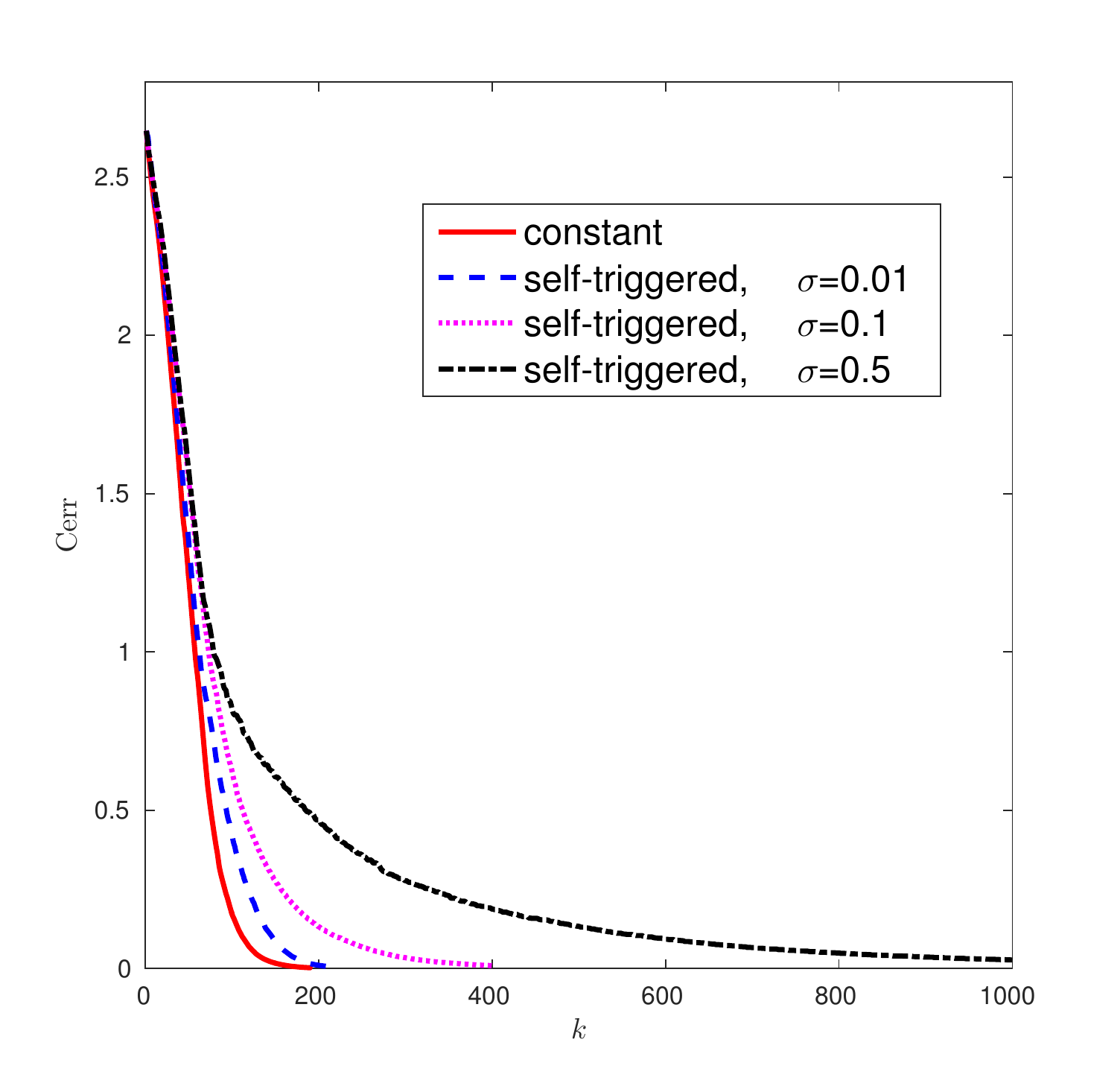}}
\subfigure[]{\includegraphics[width=0.67\columnwidth]{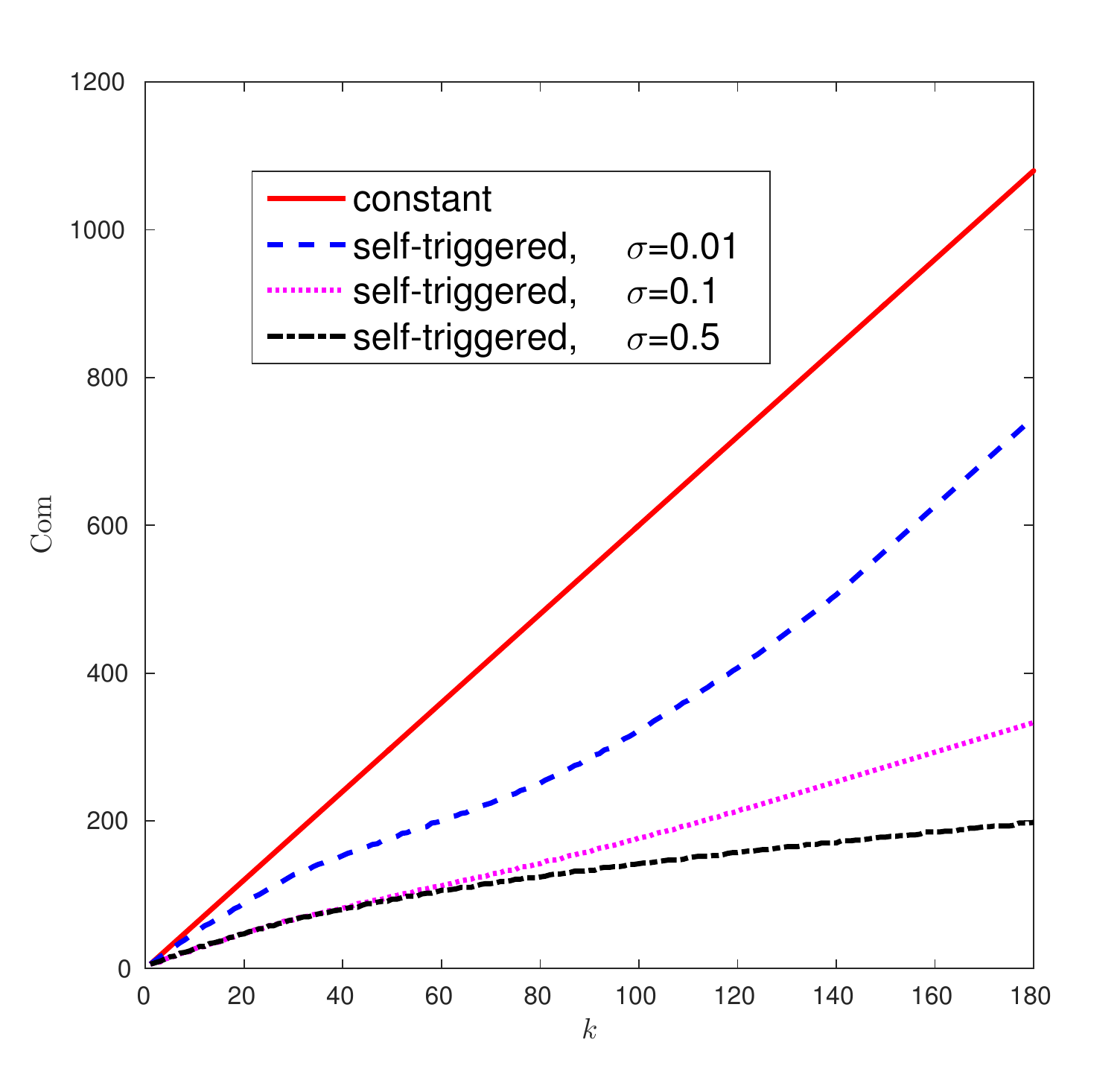}}
\caption{Comparison of convergence error and communication messages in constant and self-triggered communication strategies using the setup shown in Figure~\ref{fig:gazebo}.\label{fig:comp_fourpioneer}}}
\end{figure}

Figure \ref{fig:comp_fourpioneer}-(a) shows that the constant communication strategy converges faster than the self-triggered one with six simulated robots. Changing the tolerance parameter $\sigma$ affects the convergence time of the self-triggered strategy. The smaller the convergence tolerance $\sigma$, the faster the convergence which comes at the expense of an increased number of messages. Figure \ref{fig:comp_fourpioneer}-(b) shows communication messages for both strategies. The smaller the tolerance $\sigma$, the larger the number of messages. The convergence tolerance $\sigma$ acts as a trade-off between the communication messages and the convergence speed in the self-triggered case.

\subsection{Moving Target Case} \label{subsec:moving}

Next, we present simulation results for the realistic case of mobile, uncertain target (Section~\ref{sec:practical}). We evaluate three strategies: constant communication with centralized EKF, self-triggered communication with centralized EKF, self-triggered communication with decentralized EKF. All three algorithms were implemented in Gazebo with six simulated Pioneer robots and a simulated Pioneer target moving on a circular trajectory. 
\rev{We assume that all the robots have the same maximum linear velocity, $v_{\max}=0.2 m/s.$ We calculate the linear velocity for each robot $i$ by $v_i= \omega_i\|p_i - \hat{o}\|_2$.}
\begin{figure}[htb]
\centering
\includegraphics[width=0.67\columnwidth]{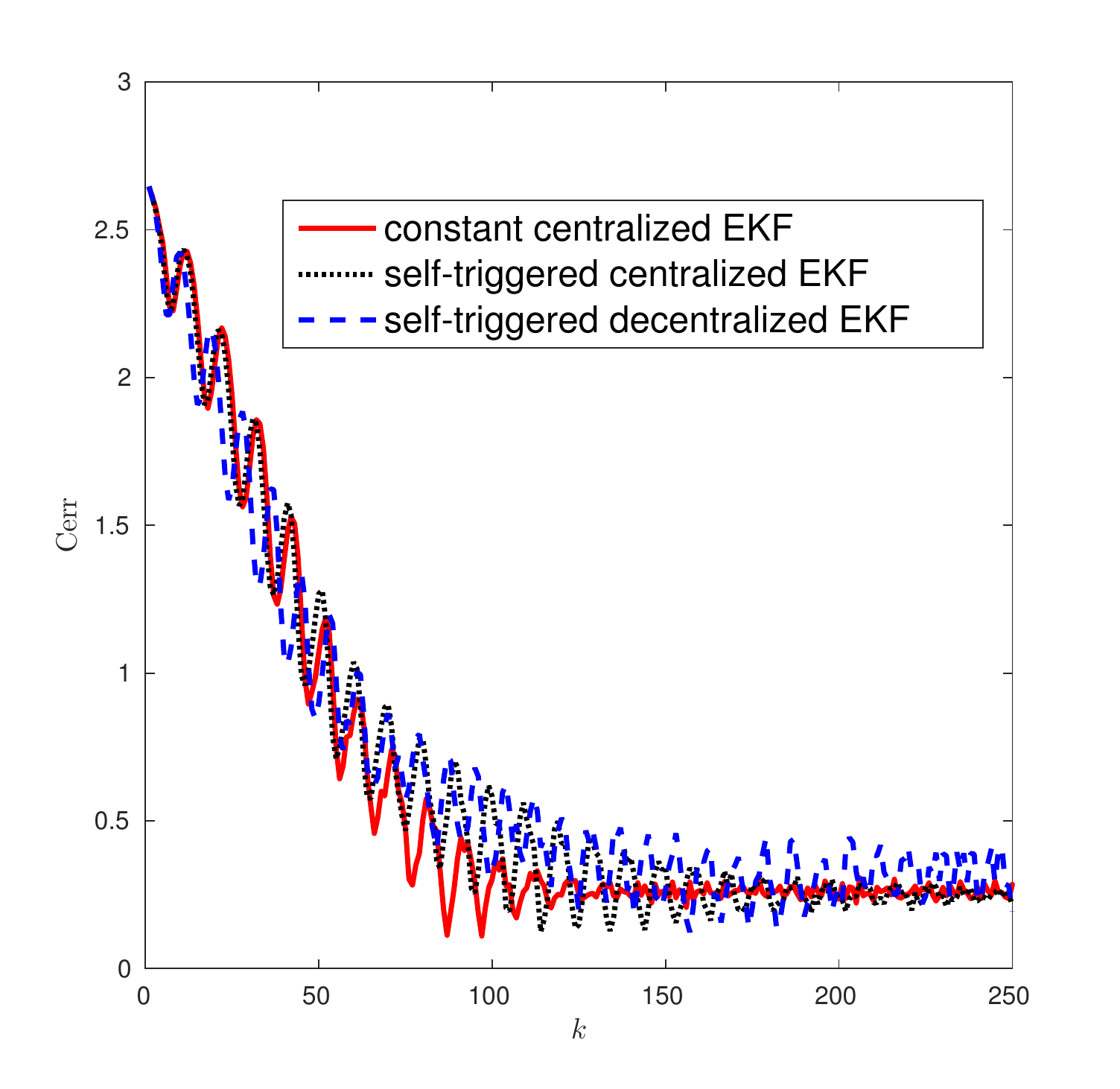}
\caption{Convergence error for mobile target tracking in constant communication with centralized EKF, self-triggered communication with centralized EKF, self-triggered communication with decentralized EKF.
\label{fig:moving}}
\end{figure}

\begin{figure}[htb]
\centering
\includegraphics[width=0.67\columnwidth]{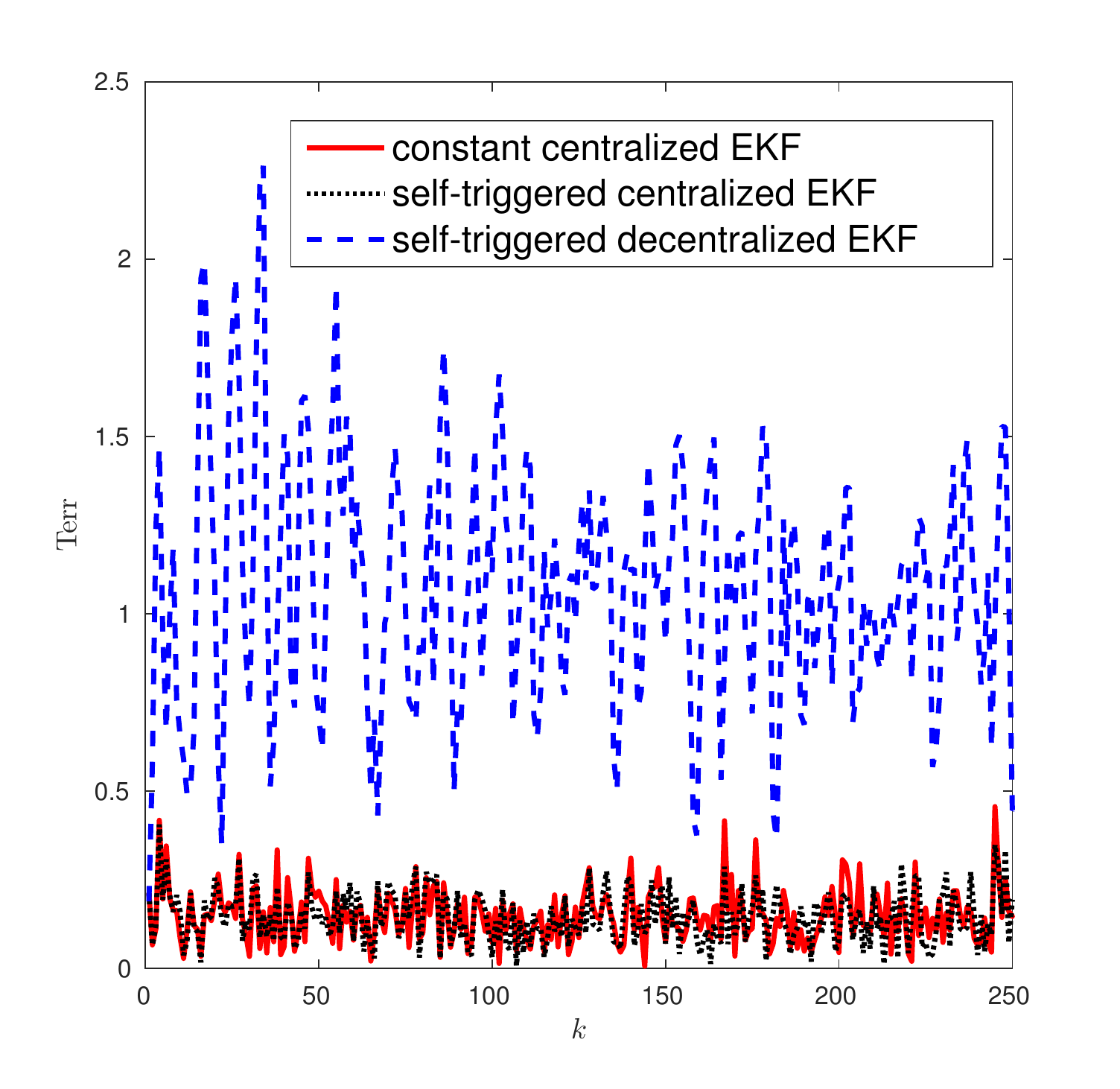}
\caption{Error in target's estimate for mobile target tracking in constant communication with centralized EKF, self-triggered communication with centralized EKF, self-triggered communication with decentralized EKF. \label{fig:movingtar}}
\end{figure}

For a moving target with $v_o=1.0 m/s$ and $\omega_o=0.6 rad/s$, Figure \ref{fig:moving} shows all three algorithms have similar tracking performance with respect to the convergence error, $\mathrm{{C}err}$, over time. However, the target estimate error $\mathrm{Terr}$ is smaller in the centralized EKF cases than the decentralized case as shown in Figure \ref{fig:movingtar}. The target estimate error is defined as:
\begin{equation*}
\mathrm{Terr}=\|\hat{o}-o\|,
\end{equation*}
for the centralized case with $\hat{o}$ indicating the centralized estimate of the target, and 
\begin{equation*}
\mathrm{Terr}=\frac{\sum_{i=1}^{N}\|\hat{o}_i-o\|}{N},
\end{equation*}
for the decentralized case with $\hat{o}_i$ indicating the target estimate from each robot $i$. 

\begin{figure*}[htb]
\centering
\subfigure[$r=\frac{0.6 m/s}{0.6 rad/s}=1m$]{\includegraphics[width=0.5\columnwidth]{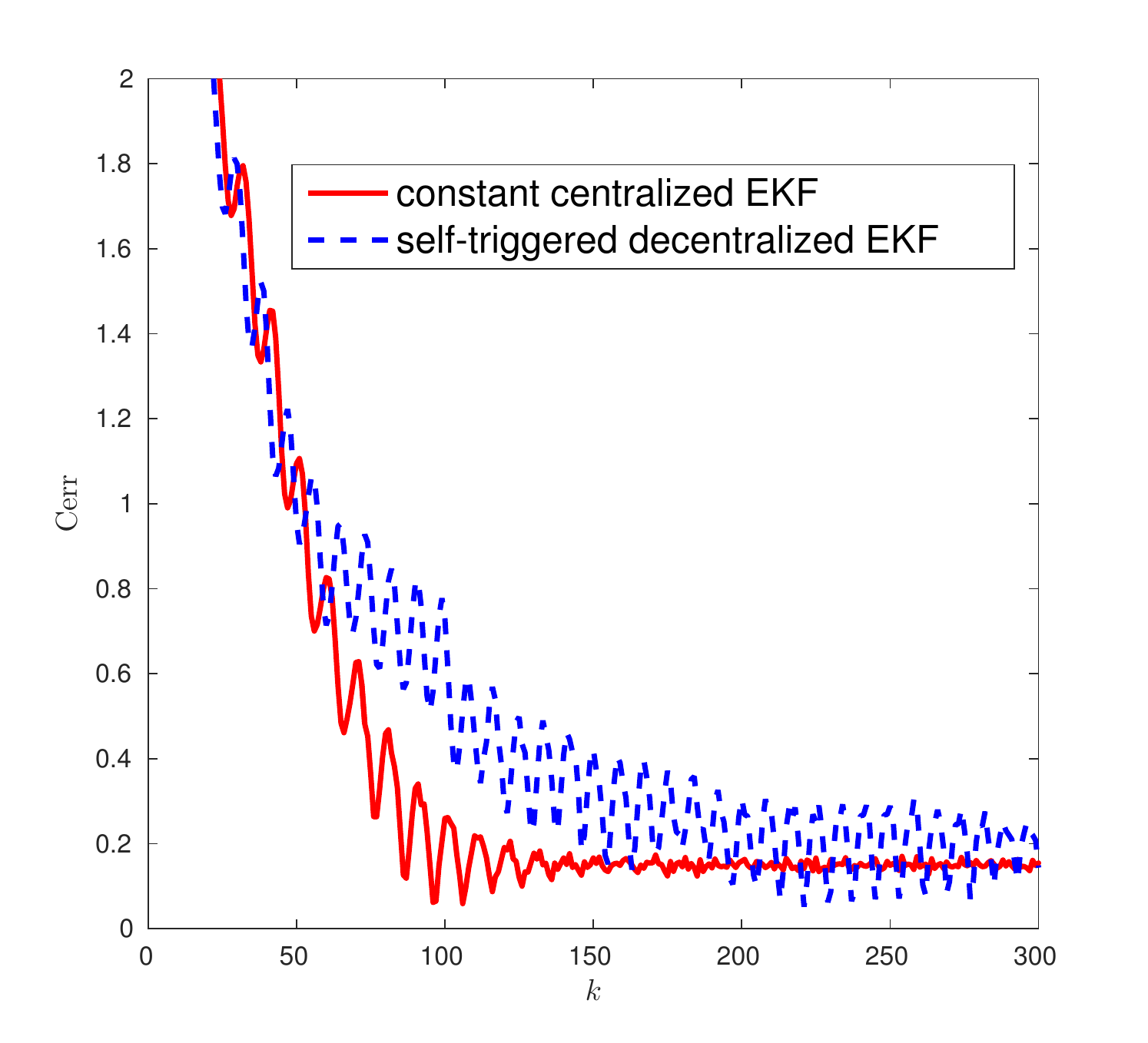}}
\subfigure[$r=\frac{1.0 m/s}{0.6 rad/s}\simeq 1.67m$]{\includegraphics[width=0.5\columnwidth]{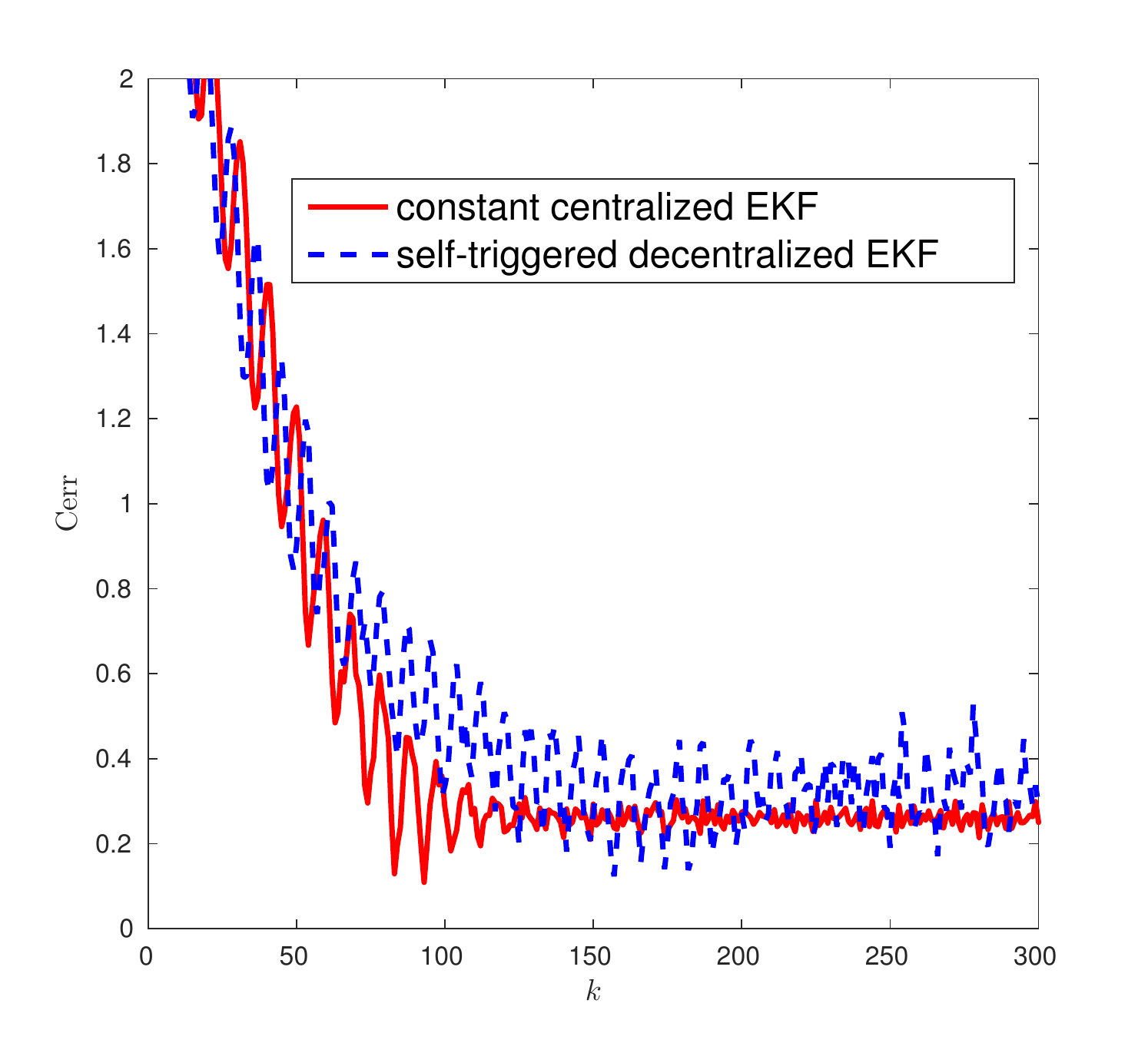}}
\subfigure[$r=\frac{1.0 m/s}{0.4 rad/s}=2.5m$]{\includegraphics[width=0.5\columnwidth]{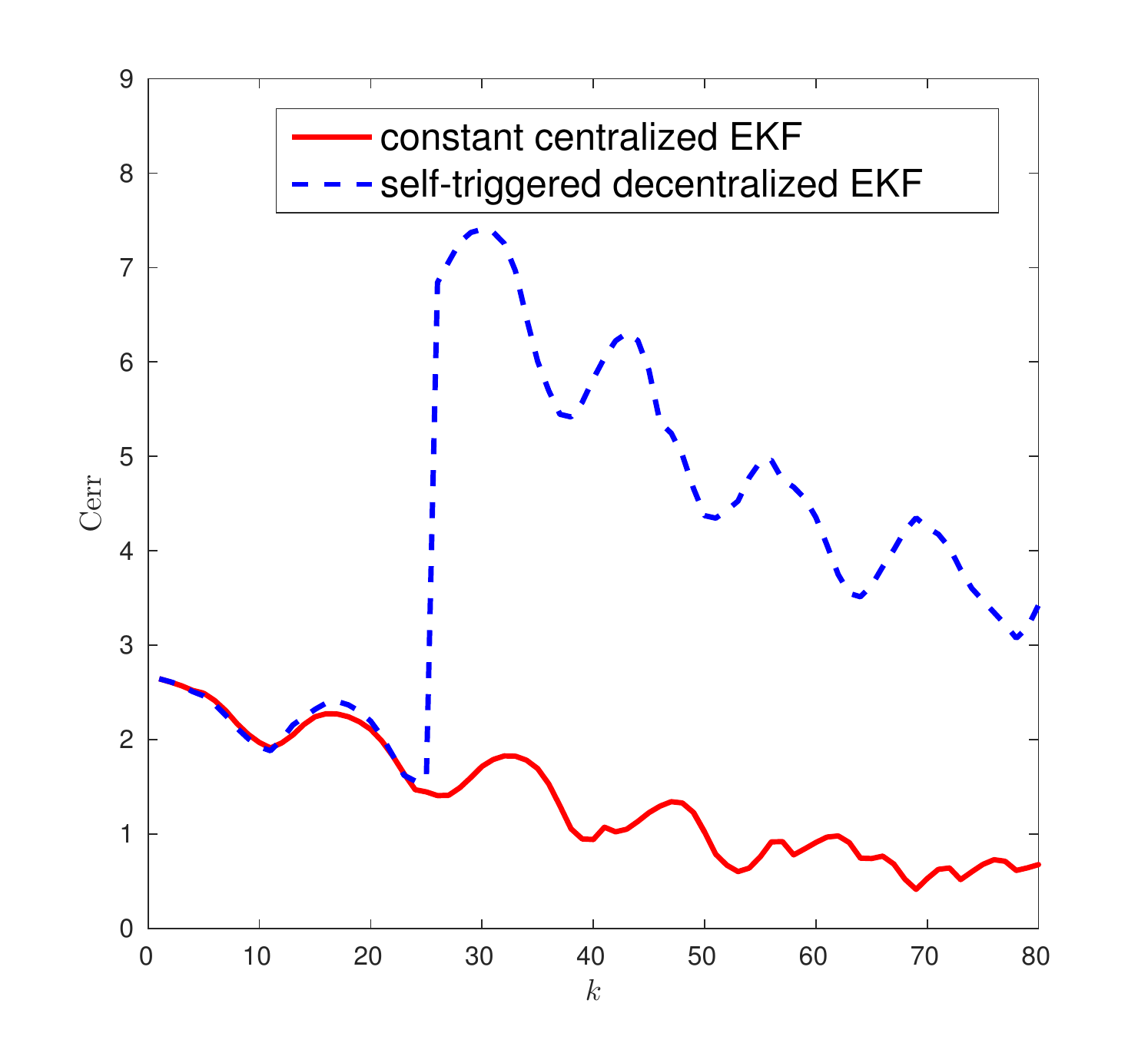}}
\subfigure[$r=\frac{1.0 m/s}{0.8 rad/s}=1.25m$]{\includegraphics[width=0.5\columnwidth]{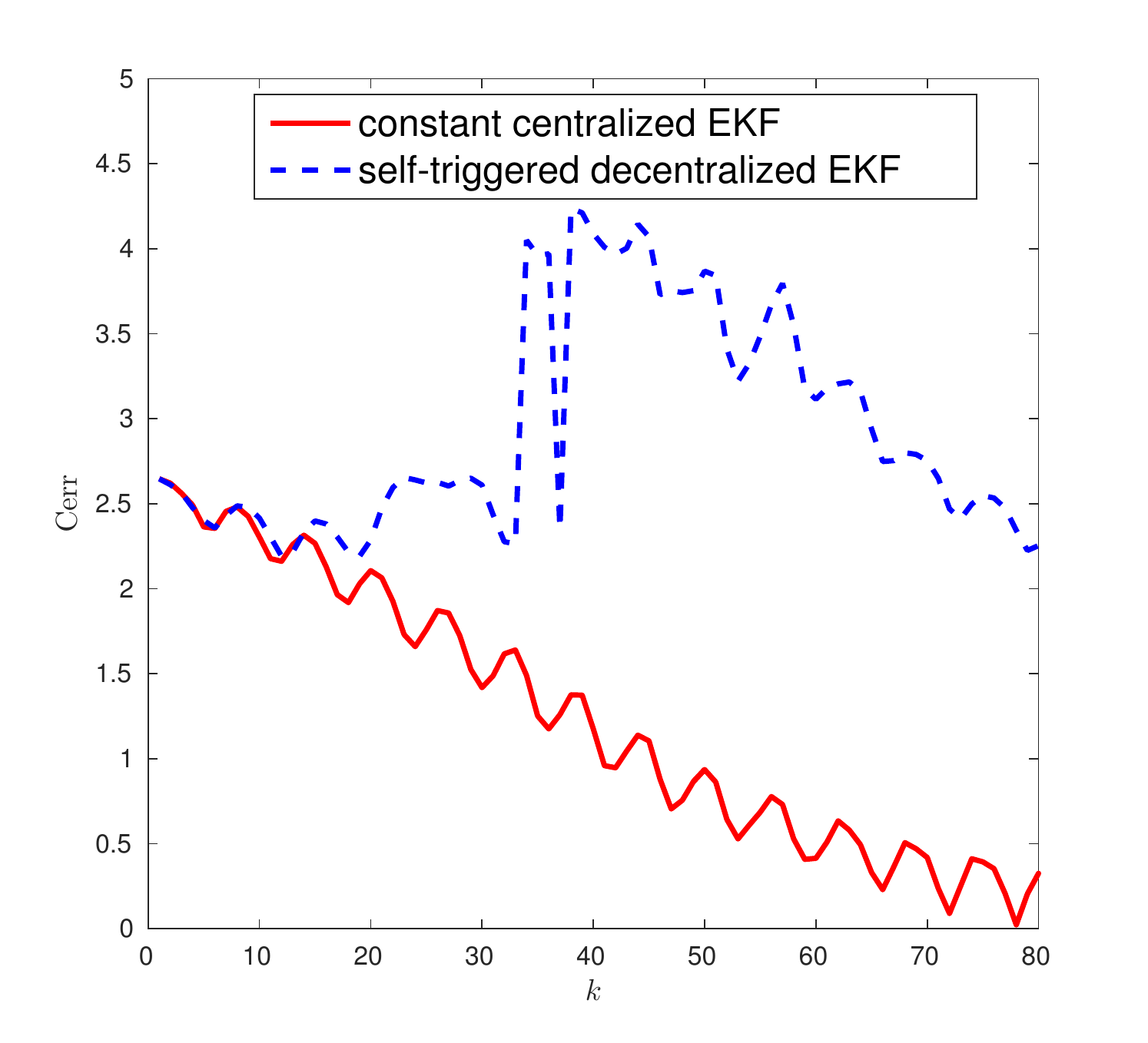}}
\caption{Comparison of $\mathrm{{C}err}$ for mobile target tracking with constant centralized EKF and self-triggered decentralized EKF w.r.t. the radius and velocities ($v$ and $\omega$) of the moving target.\label{fig:compare_good_bad}}
\end{figure*}


Figure~\ref{fig:compare_good_bad} shows the tracking performance of the self-triggered communication decentralized EKF strategy in relation to the baseline constant communication centralized EKF strategy as a function of the linear and angular velocities of the target's motion. We observe that the performance of the self-triggered strategy is comparable to the baseline algorithm except when the target moves in a large circle (Figure~\ref{fig:compare_good_bad}-(c)) and when the target moves too fast (Figure~\ref{fig:compare_good_bad}-(d)).

\begin{figure*}[htb]
\centering{
\subfigure[]{\includegraphics[width=0.64\columnwidth]{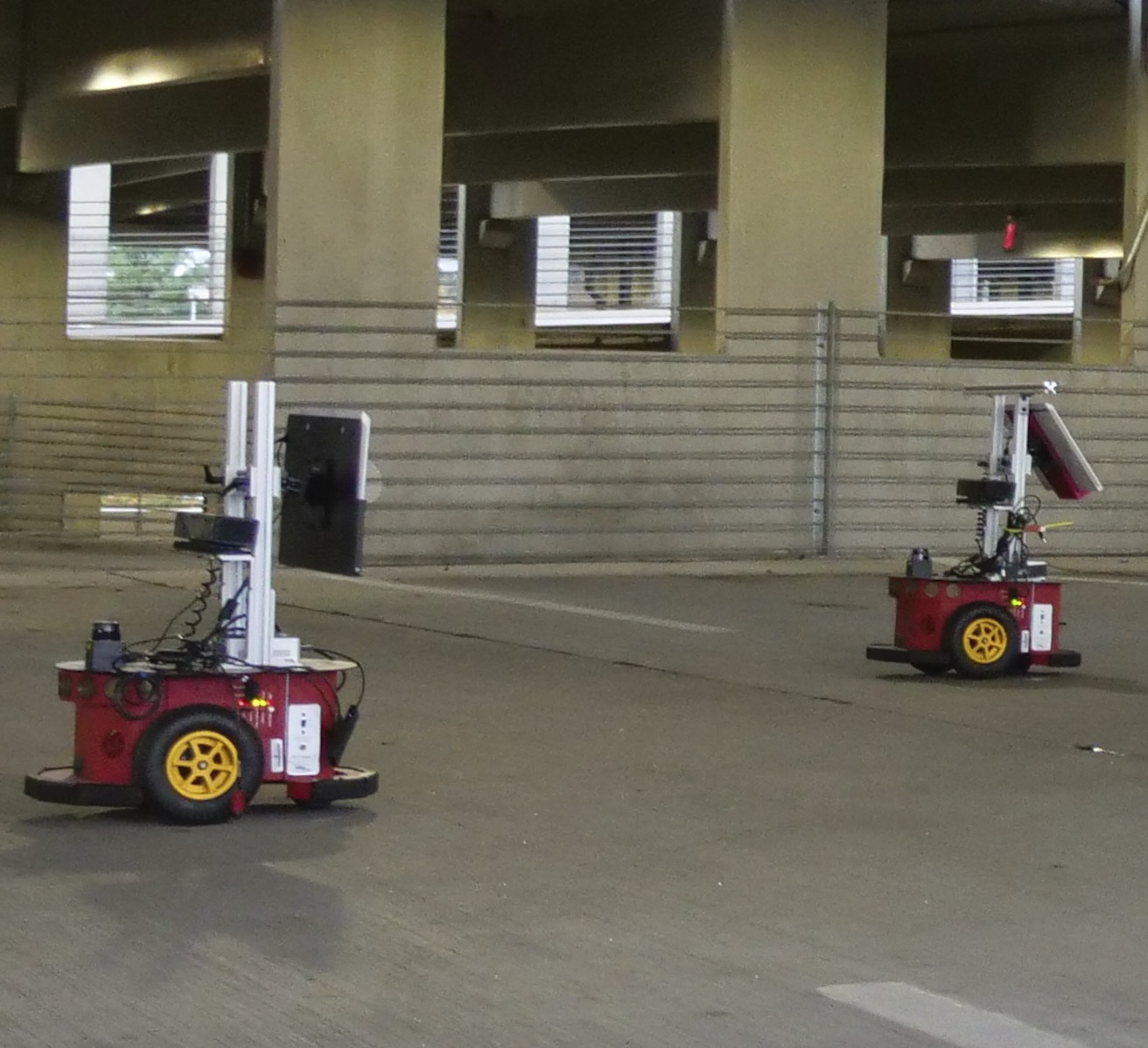}}
\subfigure[]{\includegraphics[width=0.69\columnwidth]{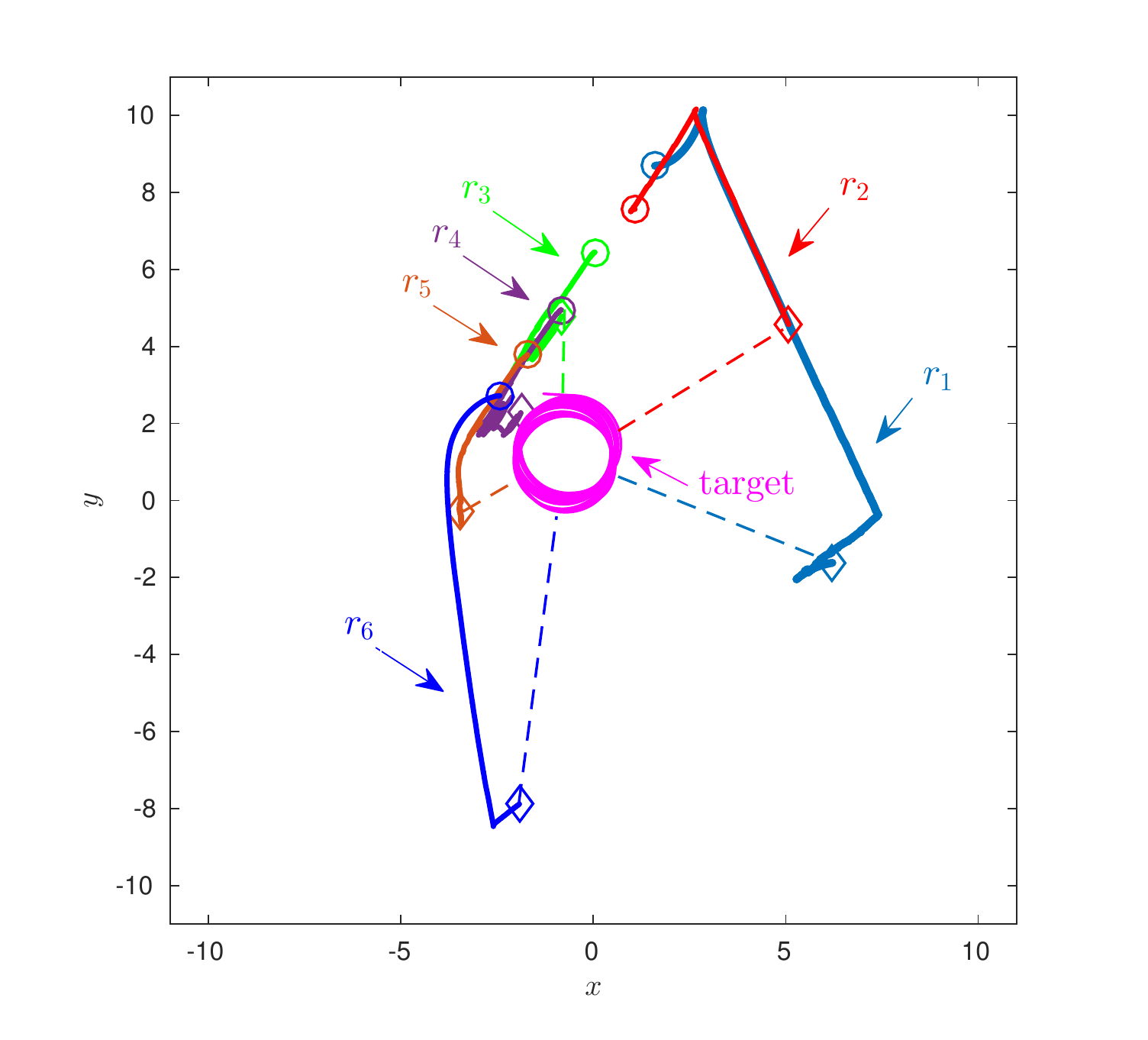}}
\subfigure[]{\includegraphics[width=0.66\columnwidth]{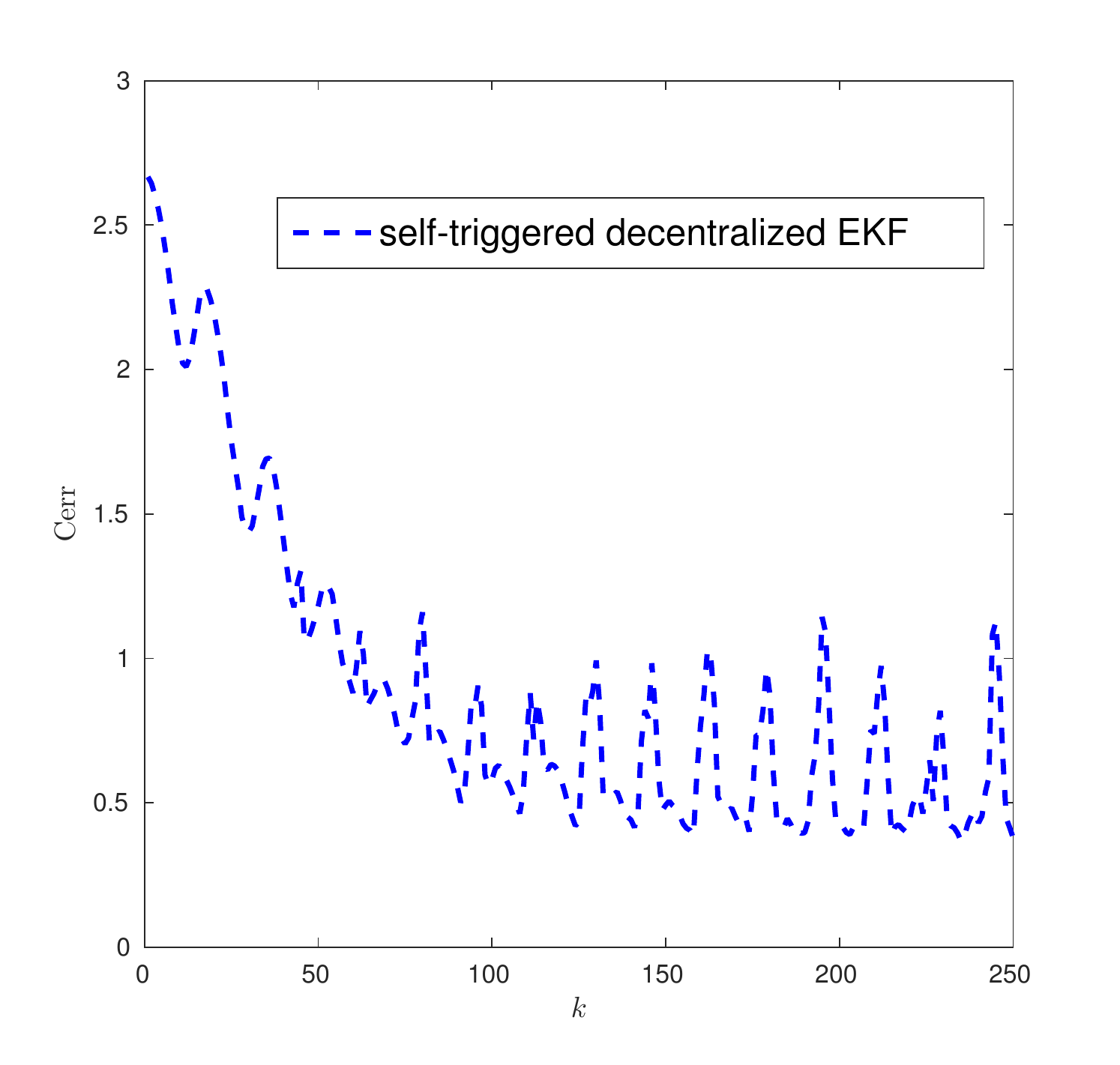}}
\caption{The trajectories of all the robots (five simulated robots, $r_1$--$r_5$, in Gazebo and two actual robots, $r_6$ and the target) and convergence error for self-triggered communication with decentralized EKF strategy. \label{fig:two_real_robot_trajectoires}}}
\end{figure*}

\subsection{Proof-of-Concept Experiment}
To further verify the tracking performance of the self-triggered decentralized EKF strategy, we also conducted a proof-of-concept mixed reality experiment. Due to limited resources, we used five simulated Pioneer 3DX robots ( $r_1\sim r_5$) cooperating with one real Pioneer 3DX robot ($r_6$) to track one real Pioneer 3DX target moving with $v_o=1.0 m/s$ and $\omega_o=0.6 rad/s$. The initial deployment for all seven robots is the same as Gazebo experiment (Figure~\ref{fig:gazebo}). The two real Pioneer robots (robot 6 and target) and the trajectories of all robots during tracking are shown in Figure~\ref{fig:two_real_robot_trajectoires}-(a) and (b). Figure~\ref{fig:two_real_robot_trajectoires}-(c) shows the self-triggered communication decentralized EKF strategy achieves a comparable tracking performance w.r.t. the convergence error. The video showing all the simulations and experiments is available online.\footnote{\url{https://youtu.be/UcsRCc9cfns}}

\section{Discussion and Conclusion}	\label{sec:conc}
In this paper, we investigated the problem of active target tracking where each robot controls not only its own positions but also decides when to communicate and exchange information with its neighbors. We focused on a simpler target tracking scenario, first studied in reference~\cite{martinez2006optimal}. We applied a self-triggered coordination strategy that asymptotically converges to a uniform configuration around the target while reducing the number of communication to less than $30\%$ of a constant strategy. We find that the self-triggered strategy performs comparably with the constant communication strategy. Future work includes extending the self-triggered strategy to decide not only when to communicate information, but also when to obtain measurements and which robots to communicate with. We conjecture that the latter question is crucial for better performance while tracking mobile targets. \rev{The self-triggered strategy can also be applied to other domains with networked controllers, e.g., for optimization of the networked industrial processes~\cite{wang2017network,wang2016combined}.}



%

\appendix
\section*{Calculation of $\omega_{\max}$} \label{app:omega}

\rev{Assume the boundary of the convex environment $\partial\mathcal{Q}$ and the position (or its estimate) of the target are known.} And Assume \rev{the} robot has a  maximum speed $v_{\max}$ with which it can move on $\partial\mathcal{Q}$.  Thus, it can move as far as $d_{\max}=v_{\max}\Delta t$ in one time step $\Delta t$. We assume that $d_{\max}$ is less than the length of any edge of the polygon. Hence, a robot can cross at most one vertex per time step. Then we split the calculation of $\omega_{\max}$ into three separate cases (Figure~\ref{fig:omegamax}). 

In all cases, let $\mathcal{E}^{i}$ be the edge on which the robot is located before moving a distance of $d_{\max}$. Let $l\mathcal{E}^{i}$ be the line supporting the edge. In cases 1 and 2, we compute $\omega_{\max}$ when the robot remains on $\mathcal{E}^{i}$ after traveling $d_{\max}$, whereas in case 3 the robot goes from $\mathcal{E}^{i}$ to $\mathcal{E}^{i+1}$

\begin{figure}
\centering
\includegraphics[width=0.65\columnwidth]{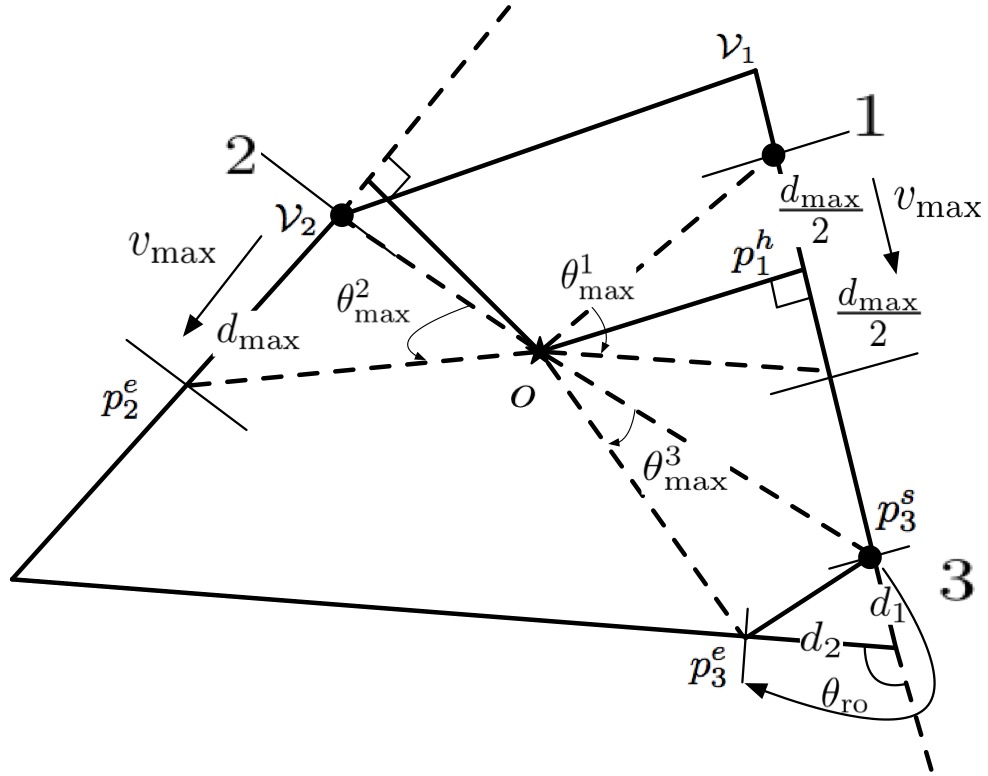}
\caption{Computing $\omega_{\max}$.\label{fig:omegamax}}
\end{figure}

\noindent\textbf{Case 1.} The orthogonal projection of the target on $l\mathcal{E}^{i}$ lies within $\mathcal{E}^{i}$.

$\omega_{\max}^{1,\mathcal{E}_{i}}$ corresponds to the case where the robot covers a maximum angular distance with respect to the target in one time step. Thus, the robot should be as close as possible to the target when it moves $d_{\max}$ on the edge. $\omega_{\max}^{1,\mathcal{E}_{i}}$ can be calculated as  $\omega_{\max}^{1,\mathcal{E}_{i}}=\frac{\theta_{\max}^{1,\mathcal{E}_{i}}}{\Delta t}$ giving $
\omega_{\max}^{1}=\min_{\mathcal{E}_{i}\in \mathcal{E}}\{\omega_{\max}^{1,\mathcal{E}_{i}}\}$. Here $\theta_{\max}^{1,\mathcal{E}_{i}}$ is the angle shown in Figure~\ref{fig:omegamax}. \rev{Since we assume $\partial\mathcal{Q}$ and the target's position (or its estimate) are known, we can calculate the length of perpendicular bisector  $|p_{1}^{h}o|$. Then, $\theta_{\max}^{1,\mathcal{E}_{i}}$ can be computed by applying Pythagorean theorem for $|p_{1}^{h}o|$ and $d_{\max}/2.$}

\noindent\textbf{Case 2.} The orthogonal projection of the target on $l\mathcal{E}^{i}$ lies outside $\mathcal{E}^{i}$.

Similar to case 1, the $\omega_{\max}$ can be computed as $\omega_{\max}^{2,\mathcal{E}_{i}}=\frac{\theta_{\max}^{2,\mathcal{E}_{i}}}{\Delta t}$ where $\omega_{\max}^{2}=\min_{\mathcal{E}_{i}\in \mathcal{E}}\{\omega_{\max}^{2,\mathcal{E}_{i}}\}$. Here, $\theta_{\max}^{2,\mathcal{E}_{i}}$ is the larger of the two angles made by the pair of lines joining target and either of the endpoint of $\mathcal{E}_{i}$ and joining target and a point $d_{\max}$ away from the corresponding endpoint. \rev{As one example in Figure~\ref{fig:omegamax}, robot starts from one vertex of the $\partial\mathcal{Q}$, $\mathcal{V}_2$, and travels $d_{\max}$ distance till the ending point $p_{2}^{e}$. Since $\partial\mathcal{Q}$ is known, we know the position of its vertex $\mathcal{V}_2$, and can compute the position of ending point, $p_{2}^{e}$ by knowing $|\mathcal{V}_2 p_{2}^{e}| = d_{\max}$ and $\partial\mathcal{Q}$. Then we can compute $|p_{2}^{e}o|$ and  $|\mathcal{V}_2o|$. By using the law of cosines, we can compute $\theta_{\max}^{2,\mathcal{E}_{i}}$, then obtain $\omega_{\max}^{2,\mathcal{E}_{i}}$.

}

\noindent\textbf{Case 3.} Robot crosses a vertex $\mathcal{V}_{i}$ within one time step.

We assume that within one time step $\Delta t$, the robot moves $d_{1}^{\mathcal{V}_{i}}$ on one edge and $d_{2}^{\mathcal{V}_{i}}$ on another edge. Since the robot must spend some time at the vertex turning in-place, we have $d_{1}^{\mathcal{V}_{i}}+d_{2}^{\mathcal{V}_{i}}<d_{\max}$. \rev{We calculate $d_2$ by
$$\frac{(d_1+d_2)}{v_{\max}}+\frac{\theta_{\mathrm{ro}}^{\mathcal{V}_{i}}}{\omega_{\mathrm{ro}}}=\Delta t.$$ where $\theta_{\mathrm{ro}}^{\mathcal{V}_{i}}$ and $\omega_{\mathrm{ro}}$ denote the rotation angle at the vertex $\mathcal{V}_{i}$ and rotational speed of the robot, which are known. Then, we show the calcualtion of  $\theta_{\max}^{3,\mathcal{V}_{i}}$ by an example in Figure~\ref{fig:omegamax} where robot starts from $p_{3}^{s}$, crosses the vertex by rotating $\theta_{\mathrm{ro}}$ and ends at $p_{3}^{e}$. Once we know $d_1$, $d_2$, $\theta_{\mathrm{ro}}$, we can use the law of cosines to calculate $|p_{3}^{s}p_{3}^{e}|$. Then apply cosine law again to $|p_{3}^{s}p_{3}^{e}|$, $|p_{3}^{s}o|$, and $|p_{3}^{e}o|$, we can compute $\theta_{\max}^{3}$. We use this procedure to calculate  $\theta_{\max}^{3,\mathcal{V}_{i}}$ at the vertex $\mathcal{V}_{i}$. Thus, the $\omega_{\max}^{3,{\mathcal{V}_{i}}}$ can be calculated by
$$\omega_{\max}^{3,\mathcal{V}_{i}}=\frac{\theta_{\max}^{3,\mathcal{V}_{i}}}{\Delta t}.$$
}
Then $\omega_{\max}^{3}$ can be specified as 

\begin{equation*}
\omega_{\max}^{3}=\min_{(\mathcal{V}_{i},d_1)}\{\omega_{\max}^{3,\mathcal{V}_{i}}\}.
\end{equation*}
\noindent Where $\mathcal{V}_{i} \in \mathcal{V}$ and $0 \le d_1 \leq \Delta t-\frac{\theta_{\mathrm{ro}}^{\mathcal{V}_{i}}}{\omega_{\mathrm{ro}}}$.

Finally, $\omega_{\max}$ can be computed as 

\begin{equation}
\omega_{\max}=\min\{\omega_{\max}^{1},\omega_{\max}^{2},\omega_{\max}^{3}\}.
\end{equation}

If $d_{\max}$ is larger than the length of one edge or the sum of lengths of several edges of the polygon, $\omega_{\max}$ can also be obtained using a similar procedure.

\section*{Acknowledgment}

This material is based upon work supported by the National Science Foundation under Grant numbers 1566247 and 1637915.

\bibliographystyle{IEEEtran.bst} 
\bibliography{refs}

\begin{IEEEbiography}[{\includegraphics[width=1in,height=1.25in,clip,keepaspectratio]{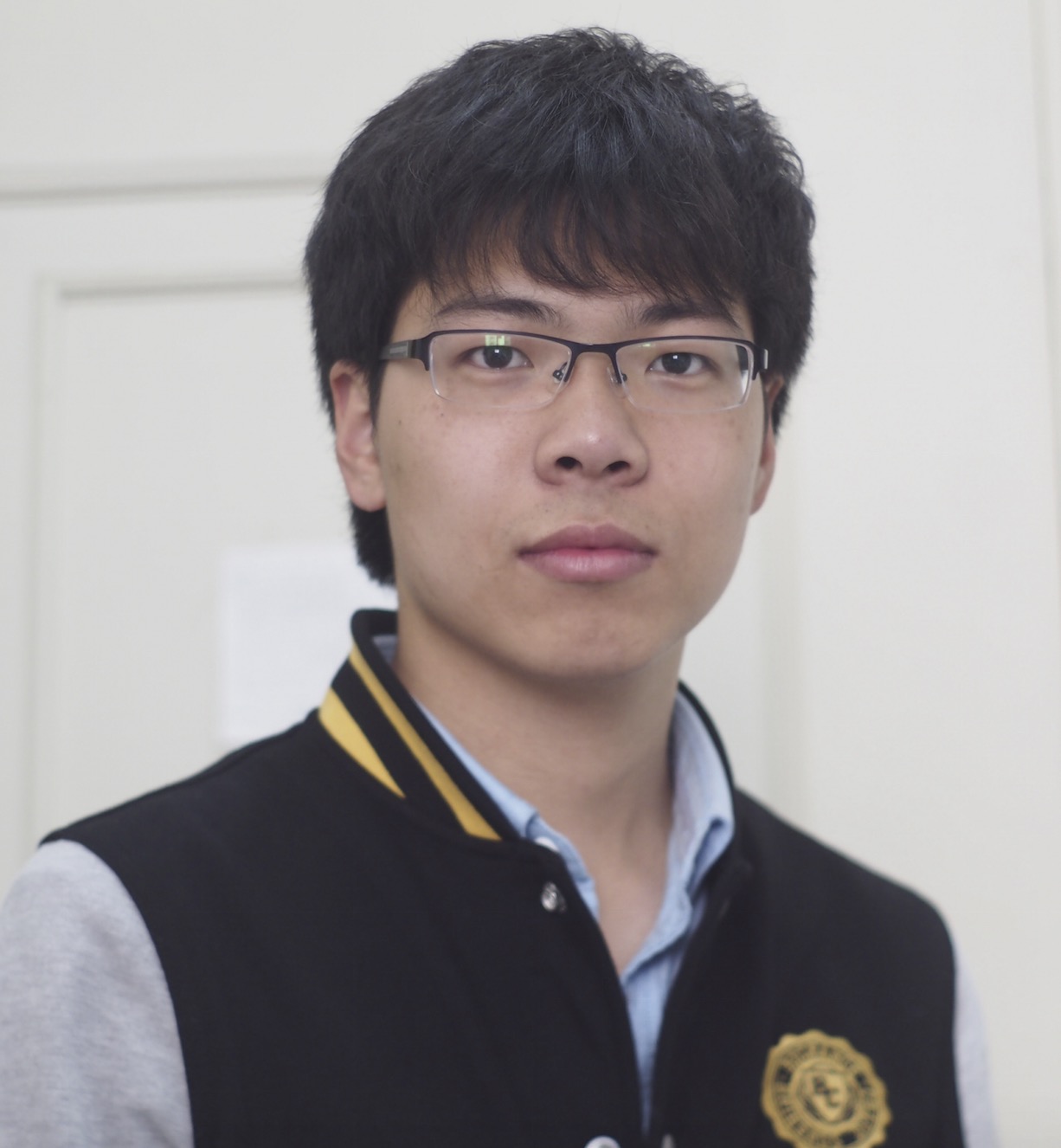}}]{Lifeng Zhou}  received the
B.S. degree in Automation from Huazhong University of Science and Technology, Wuhan, China, in 2013, the M.Sc.
degree in Automation from Shanghai Jiao Tong University, Shanghai, China, in 2016. He
is currently pursuing the Ph.D. degree in Electrical and Computer Engineering, Virginia Tech,
Blacksburg, VA, USA. 

His research interests include multi-robot coordination, event-based control, sensor assignment, and risk-averse decision making.
\end{IEEEbiography}

\begin{IEEEbiography}[{\includegraphics[width=1in,height=1.25in,clip,keepaspectratio]{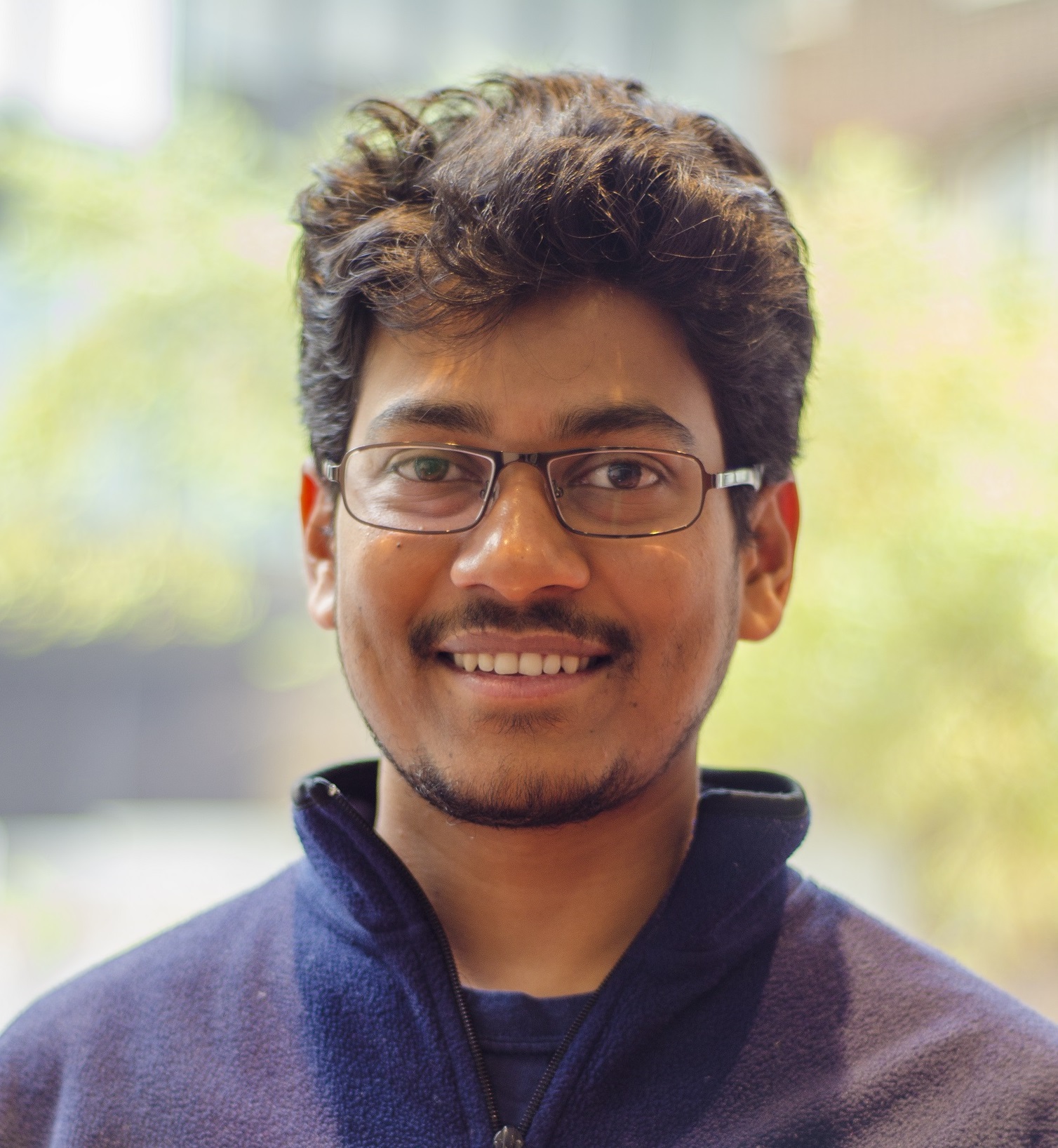}}]{Pratap Tokekar} is an Assistant Professor in the Department of Electrical and Computer Engineering at Virginia Tech. Previously, he was a Postdoctoral Researcher at the GRASP lab of University of Pennsylvania. He obtained his Ph.D. in Computer Science from the University of Minnesota in 2014 and Bachelor of Technology degree in Electronics and Telecommunication from College of Engineering Pune, India in 2008. He is a recipient of the NSF CISE Research Initiation Initiative award. His research interests include algorithmic and field robotics and  applications to precision agriculture and environmental monitoring.
\end{IEEEbiography}

\end{document}